\documentclass{article}
\usepackage[margin=1in]{geometry}
\usepackage{amsthm}
\usepackage[linesnumbered,ruled,vlined]{algorithm2e}

\usepackage{graphicx}
\usepackage{setspace}
\usepackage{caption}
\usepackage{subcaption}
\usepackage{amsmath}
\usepackage{cancel}
\usepackage{listings}
\usepackage{indentfirst}
\usepackage{algpseudocode}

\usepackage{amssymb}
\usepackage{mathtools}
\usepackage{floatrow}
\usepackage{url}
\usepackage{hyperref}
\hypersetup{
    colorlinks=true,
    linkcolor=blue,
    filecolor=magenta,      
    urlcolor=cyan,
    citecolor=blue,
}

\newtheorem{theorem}{Theorem}[section]
\newtheorem{corollary}[theorem]{Corollary}
\newtheorem{lemma}[theorem]{Lemma}
\newtheorem{proposition}[theorem]{Proposition}
\newtheorem{definition}[theorem]{Definition}
\newtheorem{remark}[theorem]{Remark}
\newtheorem{claim}[theorem]{Claim}
\newtheorem{conjecture}[theorem]{Conjecture}
\usepackage[dvipsnames]{xcolor}
\usepackage{thm-restate}
\usepackage[normalem]{ulem}

\newcommand{\TV}{d_{\mathrm{TV}}}
\newcommand{\ball}[3]{\mathcal{B}\left( #1, #2, #3 \right)}
\newcommand{\eps}{\varepsilon}
\newcommand{\Prob}{{\mathbf{P}}}
\newcommand{\UnitGauss}{{\mathcal{G}}}
\newcommand{\dGauss}{{\mathcal{G}^{d}}}
\newcommand{\kmix}{\ensuremath{k\textnormal{\normalfont-mix}}}
\newcommand{\twomix}{\ensuremath{2\textnormal{\normalfont-mix}}}
\newcommand{\wsigma}{\widetilde{\sigma}}

\newcommand{\wtilde}{\widetilde}
\newcommand{\what}{\widehat}

\newcommand{\bR}{\mathbb{R}}
\newcommand{\bN}{\mathbb{N}}
\newcommand{\bZ}{\mathbb{Z}}
\newcommand{\cA}{\mathcal{A}}

\newcommand{\cF}{\mathcal{F}}
\newcommand{\cG}{\mathcal{G}}
\newcommand{\cH}{\mathcal{H}}
\newcommand{\cK}{\mathcal{K}}

\newcommand{\cN}{\mathcal{N}}
\newcommand{\cX}{\mathcal{X}}
\newcommand{\cY}{\mathcal{Y}}

\newcommand{\comment}[1]{}

\newcommand{\learnmix}{{\fontfamily{cmtt}\selectfont Learn-Mixture}}
\newcommand{\stable}{{\fontfamily{cmtt}\selectfont Stable-Histogram}}

\newcommand{\DefinitionName}[1]{\label{def:#1}}
\newcommand{\Definition}[1]{Definition~\ref{def:#1}}

\newcommand{\LemmaName}[1]{\label{lem:#1}}
\newcommand{\Lemma}[1]{Lemma~\ref{lem:#1}}

\newcommand{\TheoremName}[1]{\label{thm:#1}}
\newcommand{\Theorem}[1]{Theorem~\ref{thm:#1}}

\newcommand{\PropositionName}[1]{\label{prop:#1}}
\newcommand{\Proposition}[1]{Proposition~\ref{prop:#1}}

\newcommand{\CorollaryName}[1]{\label{cor:#1}}
\newcommand{\Corollary}[1]{Corollary~\ref{cor:#1}}

\newcommand{\ClaimName}[1]{\label{claim:#1}}
\newcommand{\Claim}[1]{Claim~\ref{claim:#1}}

\newcommand{\EquationName}[1]{\label{eq:#1}}
\newcommand{\Equation}[1]{Eq.~\eqref{eq:#1}}

\author{
  Ishaq Aden-Ali\thanks{Department of Computing and Software, McMaster University.  \texttt{adenali@mcmaster.ca}.}
  \and
  Hassan Ashtiani\thanks{Department of Computing and Software, McMaster University. \texttt{zokaeiam@mcmaster.ca}. }
  \and
  Christopher Liaw\thanks{Department of Computer Science, University of Toronto. \texttt{cvliaw@cs.toronto.edu}. }
}

\makeatletter
\renewcommand{\paragraph}{%
  \@startsection{paragraph}{4}%
  {\z@}{1.25ex \@plus 1ex \@minus .2ex}{-1em}%
  {\normalfont\normalsize\bfseries}%
}
\makeatother

\title{Privately Learning Mixtures of Axis-Aligned Gaussians}
\date{\today}

\begin{document}

\maketitle

\begin{abstract}
    We consider the problem of learning mixtures of  Gaussians under the constraint of approximate differential privacy.
    We prove that $\wtilde{O}(k^2 d \log^{3/2}(1/\delta) / \alpha^2 \eps)$ samples are sufficient to learn a mixture of $k$ axis-aligned Gaussians in $\bR^d$ to within total variation distance $\alpha$ while satisfying $(\eps, \delta)$-differential privacy.
    This is the first result for privately learning mixtures of unbounded axis-aligned (or even unbounded univariate) Gaussians.
    If the covariance matrices of each of the Gaussians is the identity matrix, we show that $\wtilde{O}(kd/\alpha^2 + kd \log(1/\delta) / \alpha \eps)$ samples are sufficient.
    
    Recently, the ``local covering'' technique of Bun, Kamath, Steinke, and Wu~\cite{BunKSW19} has been successfully used for privately learning high-dimensional Gaussians with a known covariance matrix and extended to privately learning general high-dimensional Gaussians by Aden-Ali, Ashtiani, and Kamath~\cite{Aden-AliAK21}. Given these positive results, this approach has been proposed as a promising direction for privately learning mixtures of Gaussians. Unfortunately, we show that this is not possible.
    
    We design a new technique for privately learning mixture distributions.
    A class of distributions $\cF$ is said to be list-decodable if there is an algorithm that, given ``heavily corrupted'' samples from $f\in \cF$, outputs a list of distributions, $\what{\cF}$, such that one of the distributions in $\what{\cF}$ approximates $f$. We show that if $\cF$ is privately list-decodable, then we can privately learn mixtures of distributions in $\cF$. 
    Finally, we show axis-aligned Gaussian distributions are privately list-decodable, thereby proving mixtures of such distributions are privately learnable.
\end{abstract}

\section{Introduction}\label{sec:intro}
The fundamental problem of \emph{distribution learning} concerns the design of algorithms (i.e.,~\emph{estimators}) that, given samples generated from an unknown distribution $f$, output an ``approximation'' of $f$.
While the literature on distribution learning is vast and has a long history dating back to the late nineteenth century, the problem of distribution learning under privacy constraints is relatively new and unexplored.

In this paper, we work with the notion of differential privacy which was introduced by Dwork et al.~\cite{DworkMNS06} as a rigorous and practical notion of data privacy.
Roughly speaking, differential privacy guarantees that no single data point can influence the output of an algorithm too much, which intuitively provides privacy by ``hiding'' the contribution of each individual.
Differential privacy is the de facto standard for modern private analysis which has seen widespread impact in both industry and government \cite{ErlingssonPK14,BittauEMMRLRKTS17, DingKY17, AppleDP17, DajaniLSKRMGDGKKLSSVA17}.

In recent years, there has been a flurry of activity in differentially private distribution learning.
A number of techniques have been developed in the literature for this problem.
In the pure differentially private setting, Bun et al.~\cite{BunKSW19} recently introduced a method to learn a class of distributions when the class admits a finite cover, i.e.~when the entire class of distributions can be well-approximated by a finite number of representative distributions.
In fact, they show that this is an exact characterization of distributions which can be learned under pure differential privacy in the sense that a class of distributions is learnable under pure differential privacy if and only if the class admits a finite cover \cite{HardtT10, BunKSW19}.
As a consequence of this result, they obtained pure differentially private algorithms for learning Gaussian distributions provided that the mean of the Gaussians are bounded \emph{and} the covariance matrix of the Gaussians are spectrally bounded.\footnote{When we say that a matrix $\Sigma$ is spectrally bounded, we mean that there are $0 < a_1 \leq a_2$ such that $a_1 \cdot I \preceq \Sigma \preceq a_2 \cdot I$.}
Moreover, such restrictions on the Gaussians are necessary under the constraint of pure differential privacy.

One way to remove the requirement of having a finite cover is to relax to a weaker notion of privacy known as approximate differential privacy.
With this notion, Bun et al.~\cite{BunKSW19} introduced another method to learn a class of distributions that, instead of requiring a finite cover, requires a ``locally small'' cover, i.e.~a cover where each distribution in the class is well-approximated by only a small number of elements within the cover.
They prove that the class of Gaussians with arbitrary mean and a fixed, known covariance matrix has a locally small cover which implies an approximate differentially private algorithm to learn this class of distributions.
Later, Aden-Ali, Ashtiani, and Kamath~\cite{Aden-AliAK21} proved that the class of mean-zero Gaussians (with no assumptions on the covariance matrix) admits a locally small cover.
This can then be used to obtain an approximate differentially private algorithm to learn the class of all Gaussians.

It is a straightforward observation that if a class of distributions admits a finite cover then the class of its mixtures also admits a finite cover.
Combined with the aforementioned work of Bun et al.~this implies a pure differentially private algorithm for learning mixtures of Gaussians with bounded mean and spectrally bounded covariance matrices.
It is natural to wonder whether an analogous statement holds for locally small covers.
In other words, if a class of distributions admits a locally small cover then does the class of mixtures also admit a locally small cover?
If so, this would provide a fruitful direction to design differentially private algorithms for learning mixtures of arbitrary Gaussians.
Unfortunately, there are simple examples of classes of distributions that admit a locally small cover yet their mixture do \emph{not}.
This leaves open the question of designing private algorithms for many classes of distributions that are learnable in the non-private setting.
One concrete open problem is for the class of mixtures of two arbitrary univariate Gaussian distributions. A more general problem is private learning of mixtures of $k$ axis-aligned (or general) Gaussian distributions.

\subsection{Main Results}
We demonstrate that it is indeed possible to privately learn mixtures of unbounded univariate Gaussians. More generally, we give sample complexity upper bounds for learning mixtures of unbounded $d$-dimensional axis-aligned Gaussians.
In the following theorem and the remainder of the paper, $n$ denotes the number of samples that is given to the algorithm.

\begin{theorem}[Informal]
    \TheoremName{main-informal}
    Let $\eps, \alpha, \beta \in (0,1)$ and $\delta \in (0, 1/n)$.
    The sample complexity of learning a mixture of $k$ $d$-dimensional axis-aligned Gaussians to $\alpha$-accuracy in total variation distance under $(\eps, \delta)$-differential privacy and success probability $1- \beta$ is
    \[
        \widetilde{O}\left(\frac{k^{2}d\log^{3/2}(1/\beta\delta)}{\alpha^{2}\eps}\right).
    \]
\end{theorem}
The formal statement of this theorem can be found in \Theorem{multivariate-pac-learner}.
We note that the condition on $\delta \in (0, 1/n)$ is standard in the differential privacy literature.
Indeed, for useful privacy, $\delta$ should be ``cryptographically small'', i.e., $\delta \ll 1/n$.

Even for the univariate case, our result is the \emph{first} sample complexity upper bound for learning mixture of Gaussians under differential privacy for which the variances are unknown and the parameters of the Gaussians may be unbounded.
In the non-private setting, it is known that $\wtilde{\Theta}(kd/\alpha^2)$ samples are necessary and sufficient to learn an axis-aligned Gaussian in $\bR^d$ \cite{SureshOAJ14,AshtianiBHLMP20}.
In the private setting, the best known sample complexity lower bound is $\Omega( d/\alpha \eps \log(d))$ under $(\eps, \delta)$-DP when $\delta \leq \wtilde{O}(\sqrt{d} / n)$~\cite{KamathLSU19}.
Obtaining improved upper or lower bounds in this setting remains an open question.

If the covariance matrix of each component of the mixture is the same and known or, without loss of generality, equal to the identity matrix, then we can improve the dependence on the parameters and obtain a result that is in line with the non-private setting.
\begin{theorem}[Informal]
    \TheoremName{identity-informal}
    Let $\eps, \alpha, \beta \in (0,1)$ and $\delta \in (0, 1/n)$.
    The sample complexity of learning a mixture of $k$ $d$-dimensional Gaussians with identity covariance matrix to $\alpha$-accuracy in total variation distance under $(\eps, \delta)$-differential privacy and success probability $1- \beta$ is
    \[
        \widetilde{O}\left(\frac{kd+\log(1/\beta)}{\alpha^2} + \frac{k d\log(1/\beta\delta)}{\alpha \eps}\right).
    \]
\end{theorem}
We relegate the formal statement and the proof of this theorem to the appendix (see Appendix~\ref{app:identity-gaussians}).
Note that the work of~\cite{NissimRS07} implies an upper bound of
$O(k^2 d^3\log^2(1/\delta) / \alpha^2 \eps^2)$ for private learning of the same class albeit in the incomparable setting of parameter estimation.

\paragraph{Comparison with locally small covers.}
While the results in~\cite{BunKSW19,Aden-AliAK21} for learning Gaussian distributions under approximate differential privacy do not yield finite-time algorithms, they do give strong information-theoretic upper bounds. This is achieved by showing that certain classes of Gaussians admit locally small covers. It is thus natural to ask whether it is possible to use this approach based on locally small covers to obtain sharper upper bounds than our main result. Unfortunately, we cannot hope to do so because it is not possible to construct locally small covers for mixture classes in general. While univariate Gaussians admit locally small covers~\cite{BunKSW19}, the following simple example shows that mixtures of univariate Gaussians do not.
\begin{proposition}[Informal version of \Proposition{impossibility}]
Every cover for the class of mixtures of two univariate Gaussians is not locally small.
\end{proposition}

\subsection{Techniques}
To prove our result, we devise a novel technique which reduces the problem of privately learning mixture distributions to the problem of private list-decodable learning of distributions.
The framework of list-decodable learning was introduced by Balcan, Blum, and Vempala \cite{BalcanBS08} and Balcan, R\"{o}glin, and Teng \cite{BalcanRT09} in the context of clustering but has since been studied extensively in the literature in a number of different contexts \cite{CharikarSV17,DiakonikolasKS18b,KarmalkarKK19,CherapanamjeriMY20,DiakonikolasKK20,RaghavendraY20, RaghavendraY20b,BakshiK21}.
The problem of list-decodable learning of distributions is as follows.
There is a distribution $f$ of interest that we are aiming to learn.
However, we do not receive samples from $f$; rather we receive samples from a \emph{corrupted} distribution $g = (1-\gamma) f + \gamma h$ where $h$ is some arbitrary distribution.
In our application, $\gamma$ will be quite close to $1$.
In other words, \emph{most} of the samples are corrupted.
The goal in list-decodable learning is to output a \emph{short} list of distributions $f_1, \ldots, f_m$ with the requirement that $f$ is close to at least one of the $f_i$'s.
The formal definition of list-decodable learning can be found in \Definition{list-decodable}.
Informally, the reduction can be summarized by the following theorem which is formalized in Section~\ref{sec:reduction}.

\begin{theorem}[Informal]
    If a class of distributions $\cF$ is privately list-decodable then mixtures of distributions from $\cF$ are privately learnable.
\end{theorem}

Roughly speaking, the reduction from learning mixtures of distribution to list-decodable learning works as follows.
Suppose that there is an unknown distribution $f$ which is a mixture of $k$ distributions $f_1, \ldots, f_k$.
A list-decodable learner would then receive samples from $f$ as input and output a short list of distributions $\what{\cF}$ so that for every $f_i$ there is some element in $\what{\cF}$ that is close to $f_i$.
In particular, some mixture of distributions from $\what{\cF}$ must be close to the true distribution $f$.
Since $\what{\cF}$ is a small finite set, the set of possible mixtures must also be relatively small.
This last observation allows us to make use of private hypothesis selection which selects a good hypothesis from a small set of candidate hypotheses \cite{BunKSW19, Aden-AliAK21}.
In Section~\ref{sec:reduction}, we formally describe the aforementioned reduction.
We note that a similar connection between list-decodable learning and learning mixture distributions was also used by Diakonikolas et al.~\cite{DiakonikolasKS18b}.
However, our reduction is focused on the private setting.

The reduction shows that to privately learn mixtures, it is sufficient to design differentially private list-decodable learning algorithms that work for (corrupted versions of) the individual mixture components. To devise list-decodable learners for (corrupted) univariate Gaussian, we utilize ``stability-based'' histograms~\cite{KorolovaKMN09,BunNS16} that satisfy approximate differential privacy.

To design a list-decodable learner for corrupted univariate Gaussians, we follow a three-step approach that is inspired by the seminal work of Karwa and Vadhan \cite{KarwaV18}.
First, we use a histogram to output a list of variances one of which approximates the true variance of the Gaussian.
As a second step, we would like to output a list of means which approximate the true mean of the Gaussian.
This can be done using histograms provided that we roughly know the variance of the Gaussian.
Since we have candidate variances from the first step, we can use a sequence of histograms where
the width of the bins of each of the histograms is determined by the candidate variances from the first step.
As a last step, using the candidate variances and means from the first two steps, we are able to construct a small set of distributions one of which approximates the true Gaussian to within accuracy $\alpha$.
In the axis-aligned Gaussians setting, we use our solution for the univariate case as a subroutine on each dimension separately. 
Now that we have a list-decodable learner for axis-aligned Gaussians, we use our reduction to obtain a private learning algorithm for learning mixtures of axis-aligned Gaussians.

\subsection{Open Problems}
Many interesting open problems remain for privately learning mixtures of Gaussians. The simplest problem is to understand the exact sample complexity (up to constants) for learning mixtures of univariate Gaussians under approximate differential privacy. We make the following conjecture based on known bounds for privately learning a single Gaussian~\cite{KarwaV18}.
\begin{conjecture}[Informal]
The sample complexity of learning a mixture of $k$, univariate Gaussians to within total variation distance $\alpha$ with high probability under $(\eps,\delta)$-DP is
\[\Theta\left(\frac{k}{\alpha^2}+\frac{k}{\alpha\eps}+\frac{\log(1/\delta)}{\eps}\right). \]
\end{conjecture}
Another wide open question is whether it is even possible to privately learn mixtures of high-dimensional Gaussians when each Gaussian can have an arbitrary covariance matrix. We believe it is possible, and make the following conjecture, again based on known results for privately learning a single high-dimensional Gaussian with no assumptions on the parameters~\cite{BunKSW19,Aden-AliAK21}.
\begin{conjecture}[Informal]
The sample complexity of learning a mixture of $k$, $d$-dimensional Gaussians to with total variation distance $\alpha$ with high probability under $(\eps,\delta)$-DP is
\[\Theta\left(\frac{kd^{2}}{\alpha^2}+\frac{kd^{2}}{\alpha\eps}+\frac{\log(1/\delta)}{\eps}\right). \]
\end{conjecture}

\subsection{Additional Related Work}\label{sec:intro-related}
Recently,~\cite{BunKSW19} showed how to learn spherical Gaussian mixtures where each Gaussian component has bounded mean under pure differential privacy. Acharya, Sun and Zhang~\cite{AcharyaSZ20} were able to obtain lower bounds in the same setting that nearly match the upper bounds of Bun, Kamath, Steinke and Wu~\cite{BunKSW19}.
Both~\cite{NissimRS07,KamathSSU19} consider differentially private learning of Gaussian mixtures, however their focus is on parameter estimation and therefore require additional assumptions such as separation or boundedness of the components.

There has been a flurry of activity on differentially private distribution learning and parameter estimation in recent years for many problem settings~\cite{NissimRS07, BunUV14, DiakonikolasHS15, SteinkeU17a, SteinkeU17b, DworkSSUV15, BunSU17, KarwaV18, KamathLSU19, CaiWZ19, BunKSW19, DuFMBG20, AcharyaSZ20, KamathSU20, BiswasDKU20, LiuKKO21}. There has also been a lot of work in the locally private setting~\cite{DuchiJW17, WangHWNXYLQ16, KairouzBR16, AcharyaSZ19, DuchiR18, DuchiR19, JosephKMW19, YeB18, GaboardiRS19}. Other work on differentially private estimation include~\cite{DworkL09, Smith11, BarberD14, AcharyaSZ18, BunS19, CanonneKMUZ19, ZhangKKW20}. For a more comprehensive review of differentially private statistics, see~\cite{KamathU20}. 

\section{Preliminaries}\label{sec:preliminaries}
For any $m \in \mathbb{N}$, $[m]$ denotes the set $\{1, 2,\dots, m\}$. Let $X \sim f$ denote a random variable $X$ sampled from the distribution $f$. Let $(X^{i})_{i=1}^m\sim f^m$ denote an i.i.d.\ random sample of size $m$ from distribution $f$. For a vector $x \in \bR^{d}$, we refer to the $i$th element of vector $x$ as $x_{i}$. For any $k\in\mathbb{N}$, we define the $k$-dimensional probability simplex to be $\Delta_{k} \coloneqq \{(w_{1},\dots,w_{k}) \in \bR^k_{\geq 0} : \textstyle\sum_{i=1}^{k}w_{i}=1\}$.
For a vector $\mu \in \bR^d$ and a positive semidefinite matrix $\Sigma$, we use $\cN(\mu, \Sigma)$ to denote the multivariate normal distribution with mean $\mu$ and covariance matrix $\Sigma$.

We define $\UnitGauss$ to be the class of univariate Gaussians and $\dGauss =  \{\cN\left(\mu,\Sigma\right) \, :\, \Sigma_{ij} = 0 \,\, \forall i \neq j \text{ and } \Sigma_{ii} > 0 \,\,\forall i\}$ to be the class of axis-aligned Gaussians.

\begin{definition}[$\alpha$-net]
    Let $(X, d)$ be a metric space. A set $N \subseteq X$ is an $\alpha$-net for $X$ under the metric $d$ if for all $x \in X$, there exists $y \in N$ such that $d(x,y) \leq \alpha$.
\end{definition}
\begin{restatable}{proposition}{alphaNet}
    \PropositionName{alphaNet}
    For any $\alpha \in (0,1]$ and $k \geq 2$, there exists an $\alpha$-net of $\Delta_k$ under the $\ell_\infty$-norm of size at most $(3/\alpha)^k$.
\end{restatable}
For completeness, a simple proof of \Proposition{alphaNet} can be found in Appendix~\ref{app:useful}.

\begin{definition}[$\kmix(\mathcal{F})$]
	Let $\mathcal{F}$ be a class of probability distributions. Then the class of $k$-mixtures of $\mathcal{F}$, written \emph{$\kmix(\mathcal{F}$)}, is defined as 
	$$
	\kmix(\mathcal{F}) ~\coloneqq~ \{\: \textstyle\sum_{i=1}^{k} w_{i}f_{i} \::\:
	(w_1,\dots,w_k)\in \Delta_k ,\, f_1,\dots,f_k\in\mathcal F \:\}.
	$$
\end{definition}

\subsection{Distribution Learning}\label{sec:preliminaries-dist-learning}

A \emph{distribution learning method} is a (potentially randomized) algorithm that, given a sequence of i.i.d.\ samples from a distribution $f$, outputs a distribution $\widehat{f}$ as an estimate of $f$. The focus of this paper is on absolutely continuous probability distributions (distributions that have a density with respect to the Lebesgue measure), so we refer to a probability distribution and its probability density function interchangeably. The specific measure of ``closeness'' between distributions that we use is the \emph{total variation (TV) distance}.
\begin{definition}
  Let $g$ and $f$ be two probability distributions defined over $\mathcal{X}$ and let $\Omega$ be the Borel sigma-algebra on $\mathcal{X}$. The \emph{total variation distance} between $g$ and $f$ is defined as 
  $$\TV(g,f) = \sup_{S \in \Omega} |\Prob_g(S)-\Prob_f(S)| =  \frac{1}{2} \int_{x \in \mathcal{X}} |g(x) - f(x)|\mathrm{d}x = \frac{1}{2}\|g - f\|_1 \in [0,1]. $$
  where $\Prob_g(S)$ denotes the probability measure that $g$ assigns to $S$. Moreover, if $\cF$ is a set of distributions over a common domain, we define $\TV(g, \cF) = \inf_{f \in \cF} \TV(g, f)$.
\end{definition}
We now formally define a PAC learner.
\begin{definition}[PAC learner]
  We say Algorithm $\cA$ is a \emph{PAC-learner} for a class of distributions $\cF$ which uses $m(\alpha,\beta)$ samples, if for every $\alpha,\beta \in (0,1)$, every $f\in\cF$, and every $n\geq m(\alpha,\beta)$ the following holds: if the algorithm is given parameters $\alpha,\beta$ and a sequence of $n$ i.i.d.\ samples from $f$ as inputs, then it outputs an approximation $\widehat{f}$ such that $\TV(f,\widehat{f})\leq \alpha$ with probability at least $1-\beta$.\footnote{The probability is over $m(\alpha, \beta)$ samples drawn from $f$ and the randomness of the algorithm.}
\end{definition}

We work with a standard additive corruption model often studied in the list-decodable setting that is inspired by the work of Huber~\cite{Huber64}. 
In this model, a sample is drawn from a distribution of interest with some probability, and with the remaining probability is drawn from an arbitrary distribution.
Our list-decodable learners take samples from these ``corrupted'' distributions as input.
\begin{definition}[$\gamma$-corrupted distributions]
\DefinitionName{CorruptedDist}
Fix some distribution $f$ and let $\gamma \in (0,1)$. We define a $\gamma$-\emph{corrupted distribution} of $f$ as as any distribution $g$ such that $$g = (1-\gamma)f+\gamma h,$$
for an arbitrary distribution $h$. We define $\cH_{\gamma}(f)$ to be the set of all $\gamma$-corrupted distributions of $f$ .
\end{definition}
\begin{remark}\label{rem:corruption}
Observe that $\cH_{\gamma}(f)$ is monotone increasing in $\gamma$, i.e.~$\cH_{\gamma}(f) \subset \cH_{\gamma'}(f)$ for all $\gamma' \in (\gamma, 1)$.
To see this, note that if $g = (1-\gamma) f + \gamma h$ then we can also rewrite $$g = (1-\gamma') f + (\gamma'-\gamma) f + \gamma h  = (1-\gamma') f + \gamma' \left(\frac{(\gamma'-\gamma)}{\gamma'} f + \frac{\gamma}{\gamma'} h\right) = (1-\gamma') f + \gamma' h',$$
where $h' = \frac{\gamma'-\gamma}{\gamma} f + \frac{\gamma}{\gamma'} h$.
Hence, $g \in C_{\gamma'}(f)$.
\end{remark}
We note that in this work, we will most often deal with $\gamma$-corrupted distribution where $\gamma$ is quite close to $1$; in other words, the vast majority of the samples are corrupted.

Now we define list-decodable learning.
In this setting, the goal is to learn a distribution $f$ given samples from a $\gamma$-corrupted distribution $g$ of $f$.
Since $\gamma$ is close to $1$, instead of finding a single distribution $\widehat{f}$ that approximates $f$, our goal is to output a list of distributions, one of which is accurate.
This turns out to be a useful primitive to design algorithms for learning mixture distributions.

\begin{definition}[list-decodable learner]\DefinitionName{list-decodable}
We say algorithm $\cA_{\textsc{List}}$ is an $L$-\emph{list-decodable learner} for a class of distributions $\mathcal{F}$ using $m_{\textsc{List}}(\alpha, \beta, \gamma)$ samples if for every $\alpha,\beta,\gamma \in (0,1)$, $n \geq m_{\emph{\text{List}}}(\alpha, \beta, \gamma)$, $f\in \mathcal{F}$, and $g\in \cH_{\gamma}(f)$, the following holds: given parameters $\alpha, \beta, \gamma$ and a sequence of $n$ i.i.d.\ samples from $g$ as inputs, $\cA_{\textsc{List}}$ outputs a set of distributions $\widetilde{\cF}$ with $|\widetilde{\cF}| \leq L$ such that with probability no less than $1-\beta$ we have $\TV(f, \widetilde{\cF}) \leq \alpha$.
\end{definition}

\subsection{Differential Privacy}
Let $X^*=\cup_{i=1}^{\infty} X^i$ be the set of all datasets of arbitrary size over a domain set $X$. 
We say two datasets $D,D' \in X^*$ are neighbours if $D$ and $D'$ differ by at most one data point.
Informally, an algorithm is differentially private if its output on neighbouring databases are similar. Formally, differential privacy (DP)\footnote{We will use the acronym DP to refer to both the terms ``differential privacy'' and ``differentially private''. Which term we are using will be clear from the specific sentence.} has the following definition.
\begin{definition}[\cite{DworkMNS06, DworkKMMN06}]\label{def:DP}
  A randomized algorithm $T : X^* \rightarrow \cY$ is
  \emph{$(\eps, \delta)$-differentially private} if for all $n\geq 1$,
  for all neighbouring datasets $D,D'\in X^n$, and for all measurable
  subsets $S\subseteq \cY$,
  $$\Pr\left[T(D)\in S\right]\leq e^\eps \Pr[T(D')\in S] + \delta\,. $$
  If $\delta = 0$, we say that $T$ is $\eps$-differentially private.
\end{definition}
We refer to $\varepsilon$-DP as \emph{pure} DP, and $(\varepsilon, \delta)$-DP for $\delta > 0$ as \emph{approximate} DP.
We make use of the following property of differentially private algorithms which asserts that adaptively composing differentially private algorithms remains differentially private.
By adaptive composition, we mean that we run a sequence of algorithms $M_1(D), \ldots, M_T(D)$ where the choice of algorithm $M_t$ may depend on the outputs of $M_1(D), \ldots, M_{t-1}(D)$.
\begin{lemma}[Composition of DP~\cite{DworkMNS06,DworkRV10}]\LemmaName{composition}
    If $M$ is an adaptive composition of differentially
    private algorithms $M_1,\dots,M_T$ then the following two statements hold:
    \begin{enumerate}
        \item If $M_1,\dots,M_T$ are $(\eps_1,\delta_1),\dots,(\eps_T,\delta_T)$-differentially private, then $M$ is $(\eps,\delta)$-differentially private for 
    $$\eps = \textstyle\sum_{t=1}^{T} \eps_{t} \quad \text{and} \quad \delta = \textstyle\sum_{t=1}^{T} \delta_{t}. $$
        \item If $M_1,\dots,M_T$ are  $(\eps_0,\delta_1),\dots,(\eps_0,\delta_T)$-differentially private for some $\eps_0 \leq 1$, then for any $\delta_{0} > 0$, $M$ is $(\eps,\delta)$-differentially private for 
    $$\eps = \eps_{0}\sqrt{6T\log(1/\delta_{0})} \quad \text{and} \quad \delta=\delta_{0}+\textstyle\sum_{t=1}^{T} \delta_t. $$
    \end{enumerate}
\end{lemma}
The first statement in \Lemma{composition} is often referred to as \emph{basic} composition and the second statement is often referred to as \emph{advanced} composition.
We also make use of the fact that post-processing the output of a differentially private algorithm does not impact privacy.
\begin{lemma}[Post Processing]\LemmaName{post-processing}
    If $M:\mathcal{X}^n \rightarrow \mathcal{Y}$ is
    $(\eps,\delta)$-differentially private, and $P:\mathcal{Y} \rightarrow \mathcal{Z}$
    is any randomized function, then the algorithm
    $P \circ M$ is $(\eps,\delta)$-differentially private.
\end{lemma}

We now define $(\eps,\delta)$-DP PAC learners and
$(\eps, \delta)$-DP $L$-List-Decodable learners.
\begin{definition}[$(\eps,\delta)$-DP PAC learner]
We say algorithm $\cA$ is an \emph{$(\eps,\delta)$-DP PAC learner} for a class of distributions $\cF$ that uses $m(\alpha,\beta,\eps,\delta)$ samples if:
\begin{enumerate}
    \item Algorithm $\cA$ is a PAC Learner for $\cF$ that uses $m(\alpha,\beta,\eps,\delta)$ samples.
    \item Algorithm $\cA$ satisfies $(\eps,\delta)$-DP.
\end{enumerate}
\end{definition}

\begin{definition}[$(\eps,\delta)$-DP list-decodable learner]
We say algorithm $\cA_{\emph{\textsc{List}}}$ is an $(\eps, \delta)$-DP $L$-list-decodable learner for a class of distributions $\mathcal{F}$ that uses $m_{\emph{\text{List}}}(\alpha,\beta,\gamma,\eps,\delta)$ samples if:
\begin{enumerate}
\item Algorithm $\cA_{\emph{\textsc{List}}}$ is a $L$-list-decodable learner for $\cF$ that uses $m_{\emph{\textsc{List}}}(\alpha,\beta,\gamma,\eps,\delta)$ samples.
\item Algorithm $\cA_{\emph{\textsc{List}}}$ satisfies $(\eps,\delta)$-DP.
\end{enumerate}
\end{definition}
\section{List-decodability and Learning Mixtures}\label{sec:reduction}
In this section, we describe our general technique which reduces the problem of private learning of mixture distributions to private list-decodable learning of distributions.
We show that if we have a differentially private list-decodable learner for a class of distributions then this can be transformed, in a black-box way, to a differentially private PAC learner for the class of \emph{mixtures} of such distributions.
In the next section, we describe private list-decodable learners for the class of Gaussians and
thereby obtain private algorithms for learning mixtures of Gaussians.

First, let us begin with some intuition in the \emph{non}-private setting.
Suppose that we have a distribution $g$ which can be written as $g = \sum_{i=1}^k \frac{1}{k} f_i$.
Then we can view $g$ as a $\frac{k-1}{k}$-corrupted distribution of $f_i$ for each $i \in [k]$.
Any list-decodable algorithm that receives samples from $g$ as input is very likely to output a candidate set $\what{\cF}$ which contains distributions that are close to $f_i$ for each $i \in [k]$.
Hence, if we let $\cK = \{ \sum_{i \in [k]} \frac{1}{k} \what{f}_i \,:\, \what{f}_i \in \what{\cF} \}$, then $g$ must be close to some distribution in $\cK$.
The only remaining task is to find a distribution in $\cK$ that is close to $g$; this final task is known as hypothesis selection and has a known solution \cite{DevroyeL01}.
We note that the above argument can be easily generalized to the setting where $g$ is a non-uniform mixture, i.e.~$g = \sum_{i=1}^k w_i f_i$ where $(w_1, \ldots, w_k) \in \Delta_k$.

The above establishes a blueprint that we can follow in order to obtain a private learner for mixture distributions.
In particular, we aim to come up with a private list-decoding algorithm which receives samples from $g$ to produce a set $\what{\cF}$.
Thereafter, one can construct a candidate set $\cK$ as mixtures of distributions from $\what{\cF}$.
Note that this step does not access the samples and therefore maintains privacy. In order to choose a good candidate from $\cK$, we make use of private hypothesis selection \cite{BunKSW19, Aden-AliAK21}.

We now formalize the above argument.
Algorithm~\ref{alg:reduction} shows how a list-decodable learner can be used as a subroutine for learning mixture distributions.
In the algorithm, we also make use of a subroutine for private hypothesis selection~\cite{BunKSW19,Aden-AliAK21}.
In hypothesis selection, an algorithm is given i.i.d.\ sample access to some unknown distribution as well as a list of distributions to pick from. The goal of the algorithm is to output a distribution in the list that is close to the unknown distribution.

\begin{lemma}[\cite{Aden-AliAK21},Theorem 27]\LemmaName{PHS}
Let $n\in\bN$. There exist an $(\eps/2)$-DP algorithm $\emph{\text{PHS}}(\eps,\alpha,\beta,\cF,D)$ with the following property: for every $\eps,\alpha,\beta\in (0,1)$, and every set of distributions $\cF = \{f_{1},\dots,f_{M}\}$, when \emph{PHS} is given $\eps,\alpha,\beta, \cF$, and a dataset $D$ of $n$ i.i.d.\ samples from an unknown (arbitrary) distribution $g$ as input, it outputs a distribution $f_{j} \in \cF$ such that \[ \TV\left(g,f_{j}\right) \leq 3\cdot \TV\left(g,\cF\right)+\alpha/2,\] with probability no less than $1-\beta/2$ so long as $$n = \Omega\left(\frac{\log(M/\beta)}{\alpha^{2}} + \frac{\log(M/\beta)}{\alpha\eps}\right).$$
\end{lemma}

We now formally relate the two problems via the theorem below. 
\begin{algorithm}
\caption{\learnmix($\alpha,\beta,\eps,\delta,k,D$).}
\label{alg:reduction}
\DontPrintSemicolon
\SetKwInOut{Input}{Input}\SetKwInOut{Output}{Output}
\setstretch{1.35}
\Input{Parameters $\alpha,\beta,\eps,\delta>0$, $k \in \bN$ and dataset $D$ of $n$ i.i.d.\ samples generated $g$.}
\Output{mixture $\widehat{g} = \textstyle\sum_{i=1}^{n}\widehat{w}_{i}\widehat{f}_{i}$.}

Split $D$ into $D_{1},D_{2}$ where $|D_{1}| = n_{1}$, $|D_{2}| = n-n_{1}$ \tcp*{$n_{1} = m_{\text{List}}\left(\frac{\eps}{2},\delta,\frac{\alpha}{18},\frac{\beta}{2k},1-\frac{\alpha}{18k}\right)$.}

$\what{\cF} = \{\what{f}_{1},\dots,\what{f}_{L}\} \leftarrow \cA_{\textsc{List}}(\alpha/18,\beta/2k,1-\alpha/18k,\eps/2,\delta, D_{1})$ \tcp*{$\left(\frac{\eps}{2},\delta\right)$-DP $L$-list-decodable learner.}\label{alg:learn-mixture-line2}

Set $\widehat{\Delta}_k$ as $(18k/\alpha)$-net of $\Delta_k$ from \Proposition{alphaNet} \label{alg:learn-mixture-line3}

Set $\cK = \{ \sum_{i=1}^k \what{w}_i \what{f}_i \,:\, \what{w} \in \what{\Delta}_k, \what{f}_i \in \cK \}$\label{alg:learn-mixture-line4} 

$\widehat{g} \leftarrow \text{PHS}(\eps/2,\alpha,\beta/2,\cK,D_{2})$\label{alg:learn-mixture-line5}

\textbf{Return} $\widehat{g}$
\end{algorithm}

\begin{theorem}\TheoremName{reduction}
Let $k\in\bN$ and $\eps, \delta \in (0,1)$.
Suppose that $\cF$ is $(\eps/2, \delta)$-DP $L$-list-decodable using $m_{\emph{\textsc{List}}}$ samples.
Then Algorithm~\ref{alg:reduction} is an
$(\eps,\delta)$-DP PAC learner for $\kmix\left(\mathcal{F}\right)$ that uses \[m(\alpha,\beta,\eps,\delta) = m_{\emph{\text{List}}}\left(\frac{\alpha}{18},\frac{\beta}{2k},1-\frac{\alpha}{18k},\frac{\eps}{2},\delta\right)+ O\left(\frac{k\log(Lk/\alpha)+\log(1/\beta)}{\alpha^{2}}+\frac{k\log(Lk/\alpha)+\log(1/\beta)}{\alpha\eps}\right)\]
samples.
\end{theorem}
\begin{proof}
We begin by briefly showing that Algorithm~\ref{alg:reduction} satisfies $(\eps, \delta)$-DP before arguing about its accuracy.

\paragraph{Privacy.}
We first prove that Algorithm~\ref{alg:reduction} is $(\eps,\delta)$-DP.  Step~\ref{alg:learn-mixture-line2} of the algorithm satisfies $(\eps/2,\delta)$-DP by the fact that $\cA_{\text{List}}$ is an $(\eps/2,\delta)$-DP $L$-list-decodable learner. Steps~\ref{alg:learn-mixture-line3} and~\ref{alg:learn-mixture-line4} maintain $(\eps/2,\delta)$-DP by post processing (\Lemma{post-processing}). Finally, step~\ref{alg:learn-mixture-line5} satisfies $(\eps/2)$-DP by \Lemma{PHS}. By basic composition (\Lemma{composition}) the entire algorithm is $(\eps,\delta)$-DP.

\paragraph{Accuracy.}
We now proceed to show that Algorithm~\ref{alg:reduction} PAC learns \kmix$(\cF)$.
In step~\ref{alg:learn-mixture-line2} of Algorithm~\ref{alg:reduction}, we use the $(\eps/2, \delta)$-DP $L$-list-decodable learner to obtain a set of distributions $\what{\cF}$ of size at most $L$.
Note that for any mixture component $f_{j}$, $g$ is a $(1-w_{j})$-corrupted distribution of $f_{j}$ since
\begin{align*}
g &=w_{j}f_{j} + \textstyle\sum_{i\not= j}w_{i}f_{i}
= w_{j}f_{j} + (1-w_{j})\textstyle\sum_{i\not= j}\frac{w_{i}f_{i}}{1-w_{j}} =w_{j}f_{j} + (1-w_{j})h,
\end{align*}
where $h = \sum_{i \neq j} \frac{w_i f_i}{1-w_j}$.

Let $N = \{i \in [k]\,:\, w_i \geq \alpha / 18k\}$ denote the set of \emph{non-negligible} components.
We first show that for any non-negligible component $i \in N$, there exists $\what{f} \in \what{\cF}$ that is close to $f_i$.

\begin{claim}
    \label{claim:reduction1}
    If $|D_1| \geq m_{\textsc{List}}(\alpha/18, \beta/2k, 1-\alpha/18k, \eps/2, \delta)$ then
    $\TV(f_i, \what{\cF}) \leq \alpha / 18$ for all $i \in N$
    with probability at least $1-\beta/2$.
\end{claim}
\begin{proof}
    Fix $i \in N$.
    Note that $1-w_i \leq 1-\alpha/18k$ so $f \in \cH_{1-\alpha/18k}(f_i)$.
    Since step~\ref{alg:learn-mixture-line2} of Algorithm~\ref{alg:reduction} makes use of a list-decodable learner, as long as $|D_1| \geq m_{\textsc{List}}( \alpha/18, \beta/2k, 1-\alpha/18k,\eps/2, \delta)$ we have $\TV(f_i, \what{\cF}) \leq \alpha / 18$ with probability at least $1-\beta/2k$.
    Since this is true for any fixed $i \in N$, a union bound gives that $\TV(f_i, \what{\cF}) \leq \alpha/18$ for all $i \in N$ with probability at least $1-\beta/2$.
\end{proof}

Steps~\ref{alg:learn-mixture-line3} and \ref{alg:learn-mixture-line4} of Algorithm~\ref{alg:reduction} constructs a candidate set $\cK$ of mixture distributions using $\what{\cF}$ and a net of the probability simplex $\Delta_k$.
The next claim shows that as long as $\TV(f_i, \what{\cF})$ is small for every non-negligible $i \in N$, $\TV(g, \cK)$ is small as well.
\begin{claim}
    \label{claim:reduction2}
    If $\TV(f_i, \what{\cF}) \leq \alpha /18$ for every $i \in N$,
    then $\TV(g, \cK) \leq \alpha / 6$.
    In addition, $|\cK| \leq \left( \frac{54Lk}{\alpha}\right)^k$.
\end{claim}
\begin{proof}
    Step~\ref{alg:learn-mixture-line3} constructs a set $\widehat{\Delta}_k$ which is an $(18k/\alpha)$-net of the probability simplex $\Delta_k$ in the $\ell_{\infty}$-norm.
    By the hypothesis of the claim,
    for each $i \in N$, there exists $\what{f}_i \in \what{\cF}$ such that $\TV(f_i, \what{f}_i)\leq \alpha/18$.
    Recall that $g = \sum_{i \in [k]} w_i f_i$.
    Let $\what{w} \in \widehat{\Delta}_k$ such that $\|\what{w} - w\|_{\infty} \leq \alpha / 18k$.
    Now let $\wtilde{g} = \sum_{i\in[k]} \what{w}_i \what{f}_i$.
    Note that $\wtilde{g} \in \cK$.
    Moreover, a straightforward calculation shows that $\TV(g, \wtilde{g}) \leq \alpha / 6$ (see \Proposition{TV-mixtures} for the detailed calculations).
    This proves that $\TV(g, \cK) \leq \alpha / 6$.
    
    Lastly, to bound $|\cK|$ we have $|\cK| \leq |\what{\cF}|^k \cdot |\widehat{\Delta}_k|$.
    Note that $|\what{\cF}| \leq L$ since it is the output of an $L$-list-decodable learner and $|\widehat{\Delta}_k| \leq (54k / \alpha)^k$ by \Proposition{alphaNet}.
    This implies the claimed bound on $|\cK|$.
\end{proof}

The only remaining step is to select a good hypothesis from $\cK$.
This is achieved using the private hypothesis selection algorithm from \Lemma{PHS} which guarantees that step~\ref{alg:learn-mixture-line5} of Algorithm~\ref{alg:reduction} returns $\what{g}$ satisfying $\TV(g, \what{g}) \leq 3 \cdot \TV(g, \cK) + \alpha / 2$ with probability $1-\beta/2$ as long as
\begin{equation}
    \EquationName{reduction}
    |D_{2}|
    = \Omega\left(\frac{\log(|\cK|/\beta)}{\alpha^{2}}+\frac{\log(|\cK|/\beta)}{\alpha\eps}\right)
    =\Omega\left(\frac{k\log(Lk/\alpha)+\log(1/\beta)}{\alpha^{2}}+\frac{k\log(Lk/\alpha)+\log(1/\beta)}{\alpha\eps}\right).
\end{equation}
Combining this with Claim~\ref{claim:reduction1}, Claim~\ref{claim:reduction2}, and a union bound, we have that with probability $1-\beta$,
\[
    \TV(g, \what{g}) \leq 3 \cdot \TV(g, \cK) + \alpha/2 \leq \alpha,
\]
where the first inequality follows from private hypothesis selection and the second inequality follows from Claim~\ref{claim:reduction1} and Claim~\ref{claim:reduction2}.

Finally, the claimed sample complexity bound follows from the samples required to construct $\what{\cF}$ (which follows from Claim~\ref{claim:reduction1}) and the samples required for private hypothesis selection which is given in \Equation{reduction}.
\end{proof}

This reduction is quite useful because it is conceptually much simpler to devise list-decodable learners for a given class $\cF$. In what follows, we will devise such list-decodable learners for certain classes and use \Theorem{reduction} to obtain private PAC learners for mixtures of these classes.

\section{Learning Mixtures of Univariate Gaussians}\label{sec:univariate}

Let $\UnitGauss$ be the class of all univariate Gaussians. In this section we consider the problem of privately learning univariate Guassian Mixtures, \kmix$(\UnitGauss)$. In the previous section, we showed that it is sufficient to design  private list-decodable learners for univariate Gaussians. As a warm-up and to build intuition about our techniques, we begin with the simpler problem of constructing private list-decodable learners for Gaussians with a single known variance $\sigma^{2}$.
In what follows, we often use ``tilde'' (e.g.~$\wtilde{M}, \wtilde{V}$) to denote sets that are meant to be \emph{coarse}, or \emph{constant}, approximations and ``hat'' (e.g.~$\what{\cF}, \what{M}, \what{V}$) to denote sets that are meant to be \emph{fine}, say $O(\alpha)$, approximations.

\subsection{Warm-up: Learning Gaussian Mixtures with a Known, Shared Variance}
In this sub-section we will construct a private list-decodable learner for univariate Gaussians with a known variance $\sigma^{2}$. A useful algorithmic primitive that we will use throughout this section and the next is the \emph{stable histogram} algorithm.
\begin{lemma}[Histogram learner~\cite{KorolovaKMN09,BunNS16}]\LemmaName{stable-histogram}
Let $n\in\mathbb{N}$, $\eta,\beta,\eps \in (0,1)$ and $\delta \in (0,1/n)$. Let $D$ be a dataset of $n$ points over a domain $\cX$.
Let $K$ be a countable index set and $\mathbf{B} = \{B_i\}_{i \in K}$ be a collection of disjoint bins defined on $\cX$, i.e.~$B_i \subseteq \cX$ and $B_i \cap B_j = \emptyset$ for $i \neq j$.
Finally, let $\overline{p}_{i} = \frac{1}{n}\cdot|D \cap B_{i}|$.
There is an $(\eps,\delta)$-DP algorithm $\text{\emph{\stable}}(\eps,\delta,\eta,\beta,D,\mathbf{B})$ that takes as input parameters $\eps,\delta,\eta,\beta$, dataset $D$ and bins $\mathbf{B}$, and outputs estimates $\{\widetilde{p}_{i}\}_{i\in K}$ such that for all $i\in K$, $$|\overline{p}_{i} - \widetilde{p}_{i}| \leq \eta,$$
with probability no less than $1-\beta$ so long as 
\[
n = \Omega\left(\frac{\log(1/\beta\delta)}{\eta\eps}\right).
\]
\end{lemma}

For any fixed $\sigma^{2} > 0$ we define $\UnitGauss_{\sigma}$ to be the set of all univariate Gaussians with variance $\sigma^{2}$. For the remainder of this section, we let $g = \cN(\mu, \sigma^2) \in \UnitGauss_{\sigma}$ and $g' \in \cH_{\gamma}(g)$. (Recall that $g' \in \cH_{\gamma}(g)$ means that $g' = (1-\gamma) g + \gamma h$ for some distribution $h$.)
Algorithm~\ref{alg:univariate-mean} shows how we privately output a list of real numbers, one of which is close to the mean of $g$ given samples from $g'$. 
\begin{algorithm}
\caption{{\fontfamily{cmtt}\selectfont Univariate-Mean-Decoder}$(\beta,\gamma,\eps,\delta,\wsigma,D)$.}
\label{alg:univariate-mean}
\DontPrintSemicolon
\SetKwInOut{Input}{Input}\SetKwInOut{Output}{Output}
\setstretch{1.35}
\Input{Parameters $\eps,\beta,\gamma \in (0,1)$, $\delta \in (0,1/n)$, $\wsigma$ and dataset $D$}
\Output{Set of approximate means $ \wtilde{M}$.}

Partition $\bR$ into bins $\mathbf{B} = \{B_{i}\}_{i\in\bN}$ where $B_{i} = ((i-0.5)\wsigma,(i+0.5)\wsigma]$.\label{alg:univariate-mean-line2}

$\{\widetilde{p}_{i}\}_{i\in\bN} \leftarrow \text{\stable}(\eps,\delta,(1-\gamma)/24,\beta/2,D,\mathbf{B})$.\label{alg:univariate-mean-line3}

$H \gets \{i\,:\, \widetilde{p}_i > (1-\gamma)/8\}$ \label{alg:univariate-mean-line4}

If $|H| > 12/(1-\gamma)$ \textbf{fail} and return $\wtilde{M} = \emptyset$ \label{alg:univariate-mean-line5}

$\wtilde{M} \gets \{i \widetilde{\sigma}\,:\, i \in H\}$

\textbf{Return} $\wtilde{M}$.\label{alg:univariate-mean-line9}
\end{algorithm}
The following lemma shows that the output of Algorithm~\ref{alg:univariate-mean} is a list of real numbers with the guarantee that at least one element in the list is close to the true mean of a Gaussian which has been corrupted.
Note that the lemma assumes the slightly weaker condition where the algorithm receives an approximation to the standard deviation instead of the true standard deviation.
This additional generality is used in the next section.
\begin{lemma}\LemmaName{univariate-mean}
Algorithm~\ref{alg:univariate-mean} is an $(\eps,\delta)$-DP algorithm such that for any $g = \cN(\mu,\sigma^{2})$ and $g'\in\cH_{\gamma}(g)$, when it is given parameters $\eps, \beta, \gamma \in (0,1)$, $\delta\in (0,1/n)$, $\wsigma\in [\sigma, 2\sigma)$ and dataset $D$ of $n$ i.i.d.\ samples from $g'$ as input, it outputs a set $\wtilde{M}$ of real numbers of size
\[
    |\wtilde{M}| \leq \frac{12}{1-\gamma}.
\]
Furthermore, with probability no less than $1-\beta$ there is an element $\wtilde{\mu} \in \wtilde{M}$ such that 
\[
    |\wtilde{\mu}-\mu| \leq \sigma,
\]
so long as
\[
n = \Omega\left(\frac{\log(1/\beta\delta)}{(1-\gamma)\eps}\right).
\]
\end{lemma}
Let us begin by gathering several straightforward observations about the algorithm.
Let $p_i = \Prob_{X \sim g'}[X \in B_i]$ be the probability that a sample drawn from $g'$ lands in bin $B_i$.
Let $\overline{p}_i = \frac{1}{n} |D \cap B_i|$ be the actual number of samples drawn from $g'$ that have landed in $B_i$.
Let $j = \lceil \mu / \wtilde{\sigma} \rfloor$.
It is a simple calculation to check that $|j\wtilde{\sigma} - \mu| \leq \sigma$.
Thus, we would like to show that $j\wtilde{\sigma} \in \wtilde{M}$ or, equivalently, that $j \in H$.
As a first step, we show that many samples actually land in bin $B_j$.
\begin{claim}
    \ClaimName{univariate1}
    If $n = \Omega( \log(1/\beta)/(1-\gamma) )$ then $\overline{p}_j\, > (1-\gamma) / 6$ with probability at least $1-\beta / 2$.
\end{claim}
\begin{proof}
    First, observe that for a bin $B_i = ((i-0.5) \wtilde{\sigma}, (i+0.5) \wtilde{\sigma}]$ and $X \sim g'$, we have (recalling \Definition{CorruptedDist}), $p_i = \Prob_{X \sim g'}[X \in B_i] \geq (1-\gamma) \Prob_{X \sim g}[X \in B_i]$.
    A fairly straightforward calculation (see \Proposition{gaussian-mean-bin}) gives that $\Prob_{X \sim g}[X \in B_j] \geq 1/3$ so that $p_j \geq (1-\gamma) / 3$.
    
    A standard Chernoff bound (\Lemma{Chernoff}) implies that $|\overline{p}_j - p_j|  <  p_j / 2$ with probability at least $1 - \beta/2$ provided $n \geq C \log(1/\beta) / (1-\gamma)$ for some constant $C > 0$.
    As $p_j \geq (1-\gamma) / 3$ this implies $\overline{p}_j  > (1-\gamma) / 6$.
\end{proof}
Next, we claim that the output of the stable histogram approximately preserves the weight of all the bins and, moreover, that the output does not have too many heavy bins.
The first assertion implies that since bin $B_j$ is heavy, the stable histogram also determines that bin $B_j$ is heavy.
The second assertion implies that the algorithm does not fail.
Let $\{\wtilde{p}_i\}_{i \in \bN}$ be the output of the stable histogram, as defined in Algorithm~\ref{alg:univariate-mean}.
\begin{claim}
    \ClaimName{univariate2}
    If $n = \Omega(\log(1/\beta \delta) / (1-\gamma) \eps)$ then with probability $1-\beta/2$, we have (i) $|\overline{p}_i - \wtilde{p}_i| \leq (1-\gamma) / 24$ for all $i \in \bN$ and
    (ii) $|H| = |\{i \in \bN \,:\, \wtilde{p}_i > (1-\gamma) / 8\}| \leq 12/(1-\gamma)$.
\end{claim}
\begin{proof}
    The first assertion directly follows from \Lemma{stable-histogram} with $\eta = (1-\gamma) / 24$.
    In the event that $|\overline{p}_i - \wtilde{p}_i| \leq (1-\gamma) / 24$, we now show that $|H| \leq 12/(1-\gamma)$.
    Note that it suffices to argue that if $i \in H$ then $\overline{p}_i > (1-\gamma) / 12$.
    Since $\sum_{i \in \bN} \overline{p}_i = 1$, this implies that $|H| \leq 12/(1-\gamma)$.
    Indeed, we argue the contrapositive.
    If $\overline{p}_i \leq (1-\gamma) / 12$ then $\widetilde{p}_{i} \leq\overline{p}_{i} + (1-\gamma)/24 \leq (1-\gamma)/8$ and, hence, $i \notin H$.
\end{proof}
With \Claim{univariate1} and \Claim{univariate2} in hand, we are now ready to prove \Lemma{univariate-mean}.
\begin{proof}[Proof of \Lemma{univariate-mean}]
We briefly prove that the algorithm is private before proceeding to the other assertions of the lemma.
\paragraph{Privacy.}
Line~\ref{alg:univariate-mean-line3} is the only part of the algorithm that looks at the data and it is $(\eps, \delta)$-DP by \Lemma{stable-histogram}.
The remainder of the algorithm can be viewed as post-processing (\Lemma{post-processing}) so it does not affect the privacy.

\paragraph{Bound on $|\wtilde{M}|$.}
For the bound on $|\wtilde{M}|$, observe that if $|H| > 12/(1-\gamma)$ then the algorithm fails so $|\wtilde{M}| \leq 12/(1-\gamma)$ deterministically.

\paragraph{Accuracy.}
Let $g, g', \mu$ be as defined in the statement of the lemma.
We now show that there exists $\wtilde{\mu} \in \wtilde{M}$ such that $|\wtilde{\mu} - \mu| \leq \sigma$.
Let $j = \lceil \mu / \wtilde{\sigma} \rfloor$.
For the remainder of the proof, we assume that $n = \Omega( \log(1/\beta \delta) / (1-\gamma) \eps)$.

\Claim{univariate1} asserts that, with probability $1-\beta/2$, we have $\overline{p}_j > (1-\gamma) / 6$.
\Claim{univariate2} asserts that, with probability $1-\beta/2$, $\wtilde{p}_j \geq \overline{p}_j - (1-\gamma) / 24$ \emph{and} that $|H| \leq 12/(1-\gamma)$. By a union bound, with probability $1-\beta$, we have that $\overline{p}_j > (1-\gamma) / 8$ and the algorithm does not fail.
This implies that $j \in H$ so $j \wtilde{\sigma} \in \wtilde{M}$.
Finally, note that $|j \wtilde{\sigma} - \mu| \leq \wtilde{\sigma}/2 \leq \sigma$ where the last inequality uses the assumption that $\wtilde{\sigma} \leq 2\sigma$.
\end{proof}

\begin{corollary}\CorollaryName{univariate-list-decoder-uniform}
For any $\eps\in (0,1)$ and $\delta\in (0,1/n)$, there is an $(\eps,\delta)$-DP $L$-list-decodable learner for $\UnitGauss_{\sigma}$ with known $\sigma > 0$ where $L = O(1/(1-\gamma)\alpha)$, and the number of samples used is 
\[
m_{\textsc{List}}(\alpha,\beta,\gamma,\eps,\delta) = O\left(\frac{\log(1/\beta\delta)}{(1-\gamma)\eps}\right).
\]
\end{corollary}
\begin{proof}
The algorithm is simple; we run {\fontfamily{cmtt}\selectfont Univariate-Mean-Decoder}$(\eps,\delta,\beta,\gamma,\sigma,D)$ and obtain the set $\wtilde{M}$.
Let $\what{M}$ be an $\alpha\sigma$-net of the set of intervals $\{[\widetilde{\mu}-\sigma,\widetilde{\mu}+\sigma] \,:\, \wtilde{\mu} \in \wtilde{M}\}$ of size $|\wtilde{M}| \cdot (2 \cdot \lceil 1/2\alpha \rceil + 1)$, i.e.~
\[
    \what{M} = \{ \widetilde{\mu} + 2j\alpha\sigma \,:\, \widetilde{\mu} \in 
    \wtilde{M}, \, j \in \{0, \pm 1, \ldots, \pm \lceil 1/2\alpha \rceil \}.
\]
We then return $\what{\cF} = \{ \cN(\widehat{\mu}, \sigma^2) \,:\, \widehat{\mu} \in \what{M}\}$.
Finally, \Lemma{univariate-mean} and post-processing (\Lemma{post-processing}) imply that the algorithm is $(\eps, \delta)$-DP while \Lemma{univariate-mean} and \Proposition{TV-univariate-gaussians} imply the accuracy guarantee.\footnote{Note that we can only use \Proposition{TV-univariate-gaussians} for target $\alpha$ as large as $2/3$. For any target $\alpha > 2/3$, we can simply run the algorithm with $\alpha = 2/3$.}
\end{proof}

Finally, we use \Corollary{univariate-list-decoder-uniform} and \Theorem{reduction} to construct an $(\eps,\delta)$-DP PAC learner for \kmix$(\UnitGauss_{\sigma})$. 

\begin{theorem}\TheoremName{univariate-pac-learner-uniform}
For any $\eps\in(0,1)$ and $\delta\in(0,1/n)$, there is an $(\eps,\delta)$-DP PAC learner for \emph{\kmix}$\left(\UnitGauss_{\sigma}\right)$ with known $\sigma > 0$ that uses $$m(\alpha,\beta,\eps,\delta) = O\left(\frac{k\log(k/\alpha)+\log(1/\beta)}{\alpha^{2}}+\frac{k\log(k/\alpha\beta\delta)}{\alpha\eps}\right) =\widetilde{O}\left(\frac{k+\log(1/\beta)}{\alpha^{2}}+\frac{k\log(1/\beta\delta)}{\alpha\eps}\right)$$ samples.
\end{theorem}

Similar ideas can also be used to privately learn the class $\kmix(\cG_1^d)$.
The details can be found in Appendix~\ref{app:identity-gaussians}.

\subsection{Learning Arbitrary Univariate Gaussian Mixtures}
In this section, we construct a list-decodable learner for $\UnitGauss$, the class of all univariate Gaussians.
First, in Algorithm~\ref{alg:univariate-variance}, we design an $(\eps, \delta)$-DP algorithm that receives samples from $g' \in \cH_{\gamma}(g)$ where $g \in \UnitGauss$ and outputs a list of candidate values for the standard deviation, one of which approximates the standard deviation of $g$ with high probability.
Then, in Algorithm~\ref{alg:gaussian-decoder}, we use Algorithm~\ref{alg:univariate-mean} and Algorithm~\ref{alg:univariate-variance} to design an $(\eps, \delta)$-DP list-decoder for $\UnitGauss$.

\subsubsection{Estimating the variance}
We begin with a method to estimate the variance.
Algorithm~\ref{alg:univariate-variance} shows how to take a set of samples and output a list of standard deviations, one of which approximates the true standard deviation up to a factor of $2$.
\begin{algorithm}
\caption{{\fontfamily{cmtt}\selectfont Univariate-Variance-Decoder}$(\beta,\gamma,\eps,\delta,D)$.}
\label{alg:univariate-variance}
\DontPrintSemicolon
\SetKwInOut{Input}{Input}\SetKwInOut{Output}{Output}
\setstretch{1.35}
\Input{Parameters $\eps,\beta,\gamma \in (0,1)$, $\delta \in (0,1/n)$, and a dataset $D$}
\Output{Set of approximate standard deviations $ \wtilde{V} = \{\widetilde{\sigma}_{1},\dots,\widetilde{\sigma}_{L}\}$.}

$Y_{k} \gets |(X^{2k}-X^{2k-1})/\sqrt{2}|$ for $k\in[n]$. \tcp*{$X^{i}$s from Dataset $D = \{X^{1},\dots,X^{2n}\}$}\label{alg:univariate-variance-line2}

$D' \gets \{Y_{1},\dots,Y_{n}\}$.\label{alg:univariate-variance-line3}

Partition $\bR_{>0}$ into bins $\mathbf{B} = \{B_{i}\}_{i\in\bZ}$ where $B_{i} = (2^{i},2^{i+1}]$.\label{alg:univariate-variance-line4}

$\{\widetilde{p}_{i}\}_{i\in\bZ} \leftarrow \text{\stable}(\eps,\delta,(1-\gamma)^{2}/24,\beta/2,D',\mathbf{B})$.\label{alg:univariate-variance-line5}

$H \gets \{i\,:\, \widetilde{p}_i > (1-\gamma)^{2}/8\}$\label{alg:univariate-variance-line6}

If $|H| > 12/(1-\gamma)^{2}$ \textbf{fail} and return $\wtilde{V} = \emptyset$\label{alg:univariate-variance-line7}

$\wtilde{V} \gets \{2^{i+1}\,:\, i\in H\}$.\label{alg:univariate-variance-line8}

\textbf{Return} $\wtilde{V}$\label{alg:univariate-variance-line9}
\end{algorithm}

\begin{lemma}\LemmaName{univariate-variance}
Algorithm~\ref{alg:univariate-variance} is an $(\eps,\delta)$-DP algorithm such that for any $g=\cN(\mu,\sigma^{2})$ and $g'\in\cH_{\gamma}(g)$, when it is given parameters $\eps,\beta, \gamma \in (0,1)$, $\delta\in (0,1/n)$ and dataset $D$ of $2n$ i.i.d.\ samples from $g'$ as input, it outputs a set $\wtilde{V}$ of positive real numbers of size
\[
    |\wtilde{V}| \leq \frac{12}{(1-\gamma)^{2}}.
\]
Furthermore, with probability no less than $1-\beta$ there is an element $\wsigma \in \wtilde{V}$ such that $$\sigma \leq \wsigma < 2\sigma,$$ so long as $$n = \Omega\left(\frac{\log(1/\beta\delta)}{(1-\gamma)^{2}\eps}\right).$$ 
\end{lemma}
The proof of \Lemma{univariate-variance} mirrors that of \Lemma{univariate-mean}.
Let $g = \cN(\mu, \sigma^2)$ and $g' \in \cH_{\gamma}(g)$.
Let $X, X' \sim g'$ and let $Y = |X - X'| / \sqrt{2}$.
For an integer $i$, let $p_i = \Prob[Y \in B_i]$ where $B_i = (2^i, 2^{i+1}]$.
Let $j$ be the (unique) integer such that $\sigma \in (2^j, 2^{j+1}]$.
\begin{claim}
    \ClaimName{univariate-var1}
    If $n = \Omega(\log(1/\beta) / (1-\gamma)^2)$ then $\overline{p}_j > (1-\gamma)^2/ 6$ with probability $1-\beta/2$.
\end{claim}
\begin{proof}
    Since, $X, X' \sim g'$ and $Y = |X - X'| / \sqrt{2}$, a straightforward calculation shows that $p_j \geq (1-\gamma)^2 / 4$ (see \Proposition{gaussian-variance-bin} and \Proposition{centered-samples-prob} for details).
    
    Next, a standard Chernoff bound (\Lemma{Chernoff}) implies that $|\overline{p}_j - p_j| <  p_j / 3$ with probability at least $1 - \beta/2$ provided $n \geq C \log(1/\beta) / (1-\gamma)^2$ for some constant $C > 0$.
    As $p_j \geq (1-\gamma)^2 / 4$ this implies $\overline{p}_j > (1-\gamma)^2 / 6$.
\end{proof}

\begin{claim}
    \ClaimName{univariate-var2}
    If $n = \Omega(\log(1/\beta \delta) / (1-\gamma)^2\eps)$ then with probability $1-\beta/2$, we have (i) $|\overline{p}_i - \wtilde{p}_i| \leq (1-\gamma)^2 / 24$ for all $i \in \bN$ and (ii) $|H| = |\{i \in \bN \,:\, \wtilde{p}_i > (1-\gamma)^2 / 8\}| \leq 12/(1-\gamma)^2$.
\end{claim}
\begin{proof}
    The first assertion directly follows from \Lemma{stable-histogram} with $\eta = (1-\gamma)^2 / 24$.
    In the event that $|\overline{p}_i - \wtilde{p}_i| \leq (1-\gamma)^2 / 24$, we now show that $|H| \leq 12/(1-\gamma)^2$.
    Note that it suffices to argue that if $i \in H$ then $\overline{p}_i > (1-\gamma)^2 / 12$.
    Since $\sum_{i \in \bN} \overline{p}_i = 1$, this implies that $|H| \leq 12/(1-\gamma)^2$.
    Indeed, we argue the contrapositive.
    If $\overline{p}_i \leq (1-\gamma)^2 / 12$ then $\widetilde{p}_{i} \leq\overline{p}_{i} + (1-\gamma)^2/24 \leq (1-\gamma)^2/12$ and, hence, $i \notin H$.
\end{proof}

Given \Claim{univariate-var1} and \Claim{univariate-var2}, we now prove \Lemma{univariate-variance}.
\begin{proof}[Proof of \Lemma{univariate-variance}]
We briefly prove that the algorithm is private before proceeding to the other assertions of the lemma.
\paragraph{Privacy.}
Line~\ref{alg:univariate-variance-line5} is the only part of the algorithm that looks at the data and it is $(\eps, \delta)$-DP by \Lemma{stable-histogram}.
The remainder of the algorithm can be viewed as post-processing (\Lemma{post-processing}) so does not affect the privacy.

\paragraph{Bound on $|\wtilde{V}|$.}
For the bound on $|\wtilde{V}|$, observe that if $|H| > 12/(1-\gamma)^2$ then the algorithm fails so $|\wtilde{V}| \leq 12/(1-\gamma)^2$ deterministically.

\paragraph{Accuracy.}
Let $g, g', \sigma$ be as defined in the statement of the lemma.
We now show that there exists $\wtilde{\sigma} \in \wtilde{V}$ such that $\wtilde{\sigma} \in [\sigma, 2\sigma)$.
Let $j$ be the unique integer such that $\sigma \in (2^j, 2^{j+1}]$.
For the remainder of the proof, we assume that $n = \Omega( \log(1/\beta \delta) / (1-\gamma)^2 \eps)$.

\Claim{univariate-var1} asserts that, with probability $1-\beta/2$, we have $\overline{p}_j > (1-\gamma)^2 / 6$.
\Claim{univariate-var2} asserts that, with probability $1-\beta/2$, $\wtilde{p}_j \geq \overline{p}_j - (1-\gamma)^2 / 24$ \emph{and} that $|H| \leq 12/(1-\gamma)^2$.
By a union bound, with probability $1-\beta$, we have that $\overline{p}_j > (1-\gamma)^2 /8$ and the algorithm does not fail.
This implies that $j \in H$ so $2^{j+1} \in \wtilde{V}$ and, by the choice of $j$, $\sigma \leq 2^{j+1} < 2\sigma$.
This completes the proof.
\end{proof}

\subsubsection{A list-decodable learner for univariate Gaussians}
Finally, in this this section, we use Algorithm~\ref{alg:univariate-mean} and Algorithm~\ref{alg:univariate-variance} to design a list-decodable learner for $\UnitGauss$.
The list-decodable learner is formally described in Algorithm~\ref{alg:gaussian-decoder}.

\begin{algorithm}
\caption{{\fontfamily{cmtt}\selectfont Univariate-Gaussian-Decoder}$(\alpha,\beta,\gamma,\eps,\delta,D)$.}
\label{alg:gaussian-decoder}
\DontPrintSemicolon
\SetKwInOut{Input}{Input}\SetKwInOut{Output}{Output}
\setstretch{1.35}
\Input{Parameters $\eps,\alpha,\beta,\gamma \in (0,1)$, $\delta \in (0,1/n)$ and a dataset $D$}
\Output{Set of approximate means $\what{M}$ and variances $\what{V}$.}

Set $T = 12/(1-\gamma)^2$\label{alg:decoder-line1}

Set $\eps ' = \eps/(2\sqrt{6T\log(2(T+1)/\delta)})$ and $\delta' = \delta/2(T+1)$\label{alg:decoder-line2}

Split $D$ into $D_{1},D_{2}$ where $|D_{1}| = n_{1}$, $|D_{2}| = n_2 = n-n_{1}$\label{alg:decoder-line3} \tcp*{$n_{1} = \Theta(\log(1/\beta\delta)/(1-\gamma)^{2}\eps)$.}

$\wtilde{V} \gets \text{{\fontfamily{cmtt}\selectfont Univariate-Variance-Decoder}}(\beta/2,\gamma,\eps/2,\delta/2,D_{1})$\label{alg:decoder-line4}

Initialize $\widehat{M} \gets \emptyset$

For $\wsigma_{i} \in \wtilde{V}$ \textbf{do}\label{alg:decoder-line5}

\quad $\wtilde{M}_{i}$ = \text{{\fontfamily{cmtt}\selectfont Univariate-Mean-Decoder}}$(\beta/2,\gamma,\eps ',\delta',\wsigma_{i},D_{2})$\label{alg:decoder-line6}

\quad $\widehat{M}_i \gets \{ \widetilde{\mu} + j \alpha \widetilde{\sigma}_i \,:\, \widetilde{\mu} \in \wtilde{M}_i,\, j \in \{0, \pm 1, \pm 2,\ldots, \pm \lceil 1/\alpha \rceil \}$\label{alg:decoder-line7net}

\quad $\what{M} \gets \what{M} \cup \what{M}_{i}$\label{alg:decoder-line8}

$C \gets \left\{\log_{2}(1+\alpha),2\log_{2}(1+\alpha),\dots,\lceil1/\log_{2}(1+\alpha)\rceil\cdot\log_{2}(1+\alpha)\right\}$ \label{alg:decoder-line9}

$\what{V} \gets \{\wsigma\cdot 2^{c-1} \::\: \wsigma\in \wtilde{V},\: c\in C\}$\label{alg:decoder-line10}

\textbf{Return} $\what{M}, \what{V}$\label{alg:decoder-line11}
\end{algorithm}
\begin{lemma}\LemmaName{univariate-gaussian}
Algorithm~\ref{alg:gaussian-decoder} is an $(\eps,\delta)$-DP algorithm such for any $g=\cN(\mu,\sigma^{2})$ and $g'\in\cH_{\gamma}(g)$, when it is given parameters $\eps,\alpha, \beta, \gamma \in (0,1)$, $\delta \in (0,1/n)$ and dataset $D$ of $n$ i.i.d.\ samples from $g'$ as inputs, it outputs a set $\what{M}$ of real numbers and a set $\what{V}$ of positive real numbers such that
\[
    |\what{M}| \leq \frac{144\cdot(2\cdot \lceil 1/\alpha\rceil+1)}{(1-\gamma)^{3}}\quad \text{and}\quad |\what{V}| \leq \frac{12\cdot\lceil\log_{1+\alpha}(2) \rceil}{(1-\gamma)^{2}}.
\]
Furthermore, with probability no less than $1-\beta$, we have the following:
\begin{enumerate}
    \item $\exists \widehat{\mu} \in \what{M}$ such that $|\widehat{\mu}-\mu| \leq\alpha\sigma$
    \item $\exists \widehat{\sigma} \in \what{V}$ such that $|\widehat{\sigma}-\sigma| \leq \alpha\sigma$
\end{enumerate}
so long as 
\[
n =\Omega\left(\frac{\log(1/\beta\delta)}{(1-\gamma)^{2}\eps}+\frac{\log(1/(1-\gamma)\beta\delta)\sqrt{\log(1/(1-\gamma)\delta)}}{(1-\gamma)^{2}\eps}\right) = \widetilde{\Omega}\left(\frac{\log^{3/2}(1/\beta\delta)}{(1-\gamma)^{2}\eps}\right).
\]
\end{lemma}

Before we prove the lemma, we make a few simple observations.
Fix $g = \cN(\mu, \sigma^2)$ and $g' \in \cH_{\gamma}(g)$.
We assume that the algorithm receives $D \sim (g')^{2n}$ as input.
\begin{claim}
    \ClaimName{univariate-decoder1}
    If $n_1 = \Omega(\log(1/\beta\delta) / (1-\gamma)^2 \eps)$ then with probability $1 - \beta/2$,
    (i) there exists $\wtilde{\sigma} \in \wtilde{V}$ such that $\wtilde{\sigma} \in [\sigma, 2\sigma)$ and
    (ii) there exists $\widehat{\sigma} \in \what{V}$ such that that $|\widehat{\sigma} - \sigma| \leq \alpha \sigma$.
\end{claim}
\begin{proof}
    \Lemma{univariate-variance} directly implies that in line~\ref{alg:decoder-line4}, with probability $1 - \beta/2$, there is some $\wtilde{\sigma} \in \wtilde{V}$ such that $\wtilde{\sigma} \in [\sigma, 2\sigma)$.
    
    For the final assertion, suppose that $\wtilde{\sigma} \in [\sigma, 2\sigma)$.
    In particular, $\log_2(2\sigma / \wtilde{\sigma}) \in (0, 1]$.
    Note that $C$ is $\log_2(1+\alpha)$-net of the interval $[0,1]$.
    Hence, there exists some $c \in C$ such that $|c - \log_2(2\sigma / \wtilde{\sigma})| \leq \log_2(1+\alpha)$.
    For such a value of $c$, we have $(\wtilde{\sigma} / \sigma) \cdot 2^{c-1} \in \left[ 1/(1+\alpha), 1+\alpha\right]$, which upon rearranging gives $\wtilde{\sigma} 2^{c-1} \in [ \sigma/(1+\alpha), \sigma(1+\alpha)]$.
    As $1/(1+\alpha) \geq 1-\alpha$, this shows that $|\wtilde{\sigma} 2^{c-1} - \sigma| \leq \alpha \sigma$.
    This completes the proof since $\wtilde{\sigma} 2^{c-1} \in \what{V}$.
\end{proof}
\begin{claim}
    \ClaimName{univariate-decoder2}
    Let $\eps', \delta'$ be as defined in Algorithm~\ref{alg:gaussian-decoder}.
    Suppose that there exists $\wtilde{\sigma}_{i} \in \wtilde{V}$ such that $\wtilde{\sigma}_{i} \in [\sigma, 2\sigma)$.
    If $n_2 = \Omega(\log(1/\beta \delta') / (1-\gamma) \eps')$ then with probability $1-\beta/2$ there exists $\widehat{\mu} \in \what{M}$ such that $|\widehat{\mu}-\mu| \leq \alpha \sigma$.
\end{claim}
\begin{proof}
    The condition that there exists $\wtilde{\sigma}_{i} \in \wtilde{V}$ such that $\wtilde{\sigma}_{i} \in [\sigma, 2\sigma)$ implies that
    one of the runs of \text{{\fontfamily{cmtt}\selectfont Univariate-Mean-Decoder}} on line~\ref{alg:decoder-line6} uses $\wsigma_{i} \in[\sigma, 2\sigma)$.
    The guarantee of \Lemma{univariate-mean} shows that with probability $1-\beta/2$, there is some $ \widetilde{\mu} \in \widetilde{M}_i$ satisfying $|\widetilde{\mu}-\mu| \leq \sigma$. 
    Finally, on line~\ref{alg:decoder-line7net}, the algorithm constructs $\what{M}_i$ which is a $(\alpha\wsigma_{i}/2)$-net of the interval $[\widetilde{\mu} - \wsigma_{i}, \widetilde{\mu} + \wsigma_{i}] \supset [\widetilde{\mu} - \sigma, \widetilde{\mu} + \sigma]$.
    Hence, there exists $\widehat{\mu} \in \widehat{M}_i$ such that $|\widehat{\mu}-\mu| \leq \alpha \wsigma / 2 < \alpha \sigma$ where the latter inequality used that $\wsigma < 2\sigma$.
    Since $\widehat{M}_i \subset \widehat{M}$, this implies the claim.
\end{proof}

\begin{proof}[Proof of \Lemma{univariate-gaussian}]
~
\paragraph{Privacy.}
We first prove that the algorithm is $(\eps,\delta)$-DP. By \Lemma{univariate-mean}, line~\ref{alg:decoder-line4} satisfies $(\eps/2,\delta/2)$-DP. The loop on line~\ref{alg:decoder-line5} runs at most $12/(1-\gamma)^2$ times since $|\wtilde{V}| \leq 12/(1-\gamma)^2$ (see \Lemma{univariate-variance}). So, by our choice of $\eps'$, $\delta'$ (line~\ref{alg:decoder-line2}) and advanced composition (\Lemma{composition}), all the iterations of line~\ref{alg:decoder-line6} collectively satisfy $(\eps/2,\delta/2)$-DP. No subsequent part of the algorithm accesses the data so by basic composition (\Lemma{composition}) and post processing~(\Lemma{post-processing}), the entire algorithm is $(\eps, \delta)$-DP.

\paragraph{Bound on $|\what{M}|$ and $|\what{V}|$.}
We now prove the claimed upper bounds on the sizes of $\what{M}$ and $\what{V}$.
First, we have $|\wtilde{V}| \leq 12/(1-\gamma)^2$ by \Lemma{univariate-variance}.
Since $|C|= \lceil 1/\log_2(1+\alpha) \rceil = \lceil \log_{1+\alpha}(2) \rceil$, this gives $|\what{V}| = |\wtilde{V}| \cdot |C| \leq 12 \cdot \lceil \log_{1+\alpha}(2) \rceil / (1-\gamma)^2$.
Next, we have that each $|\wtilde{M}_i| \leq 12/(1-\gamma)$ in Line~\ref{alg:decoder-line7net} by \Lemma{univariate-mean}, so $|\what{M}_i| \leq 12 \cdot (2\cdot \lceil 1/\alpha \rceil + 1) /(1-\gamma)$.
Hence, $|\what{M}| \leq |\wtilde{V}| \cdot 12\cdot (2\cdot \lceil 1/\alpha \rceil + 1)/(1-\gamma)  \leq 144\cdot (2\cdot \lceil 1/\alpha \rceil + 1)/(1-\gamma)^3 $.

\paragraph{Existence of $\widehat{\mu}$ and $\widehat{\sigma}$.}
\Claim{univariate-decoder1} asserts that with probability $1-\beta/2$, there is $\wtilde{\sigma} \in \wtilde{V}$ such that $\wsigma \in [\sigma, 2\sigma)$ and that there exists $\widehat{\sigma} \in \what{V}$ such that $|\widehat{\sigma} - \sigma| \leq \alpha \sigma$.
The latter statement is the bound that we asserted for $\widehat{\sigma}$ in the statement of the lemma.

Next, conditioning on the event that there exists $\wtilde{\sigma} \in \wtilde{V}$ such that $\wsigma \in [\sigma, 2\sigma)$,
\Claim{univariate-decoder2} implies that with probability $1-\beta/2$, there is some $\widehat{\mu} \in \what{M}$ such that $|\widehat{\mu} - \mu| \leq \alpha \sigma$.

To conclude, taking a union bound shows that with probability $1-\beta$, there exists $\widehat{\mu} \in \what{M}, \widehat{\sigma} \in \what{V}$ satisfying $|\widehat{\mu} - \mu| \leq \alpha \sigma$ and $|\widehat{\sigma} - \sigma| \leq \alpha \sigma$.

\paragraph{Sample complexity.}
Finally, we argue about the sample complexity.
For \Claim{univariate-decoder1}, we needed $n_1 = \Omega( \log(1/\beta \delta) / (1-\gamma)^2 \eps)$ samples and for \Claim{univariate-decoder2}, we needed $n_2 = \Omega(\log(1/\beta \delta') / (1-\gamma) \eps')$ samples.
Adding $n_1, n_2$ and plugging in the values for $\eps', \delta'$ as defined in Algorithm~\ref{alg:gaussian-decoder} gives the claimed bound on the number of samples required.
\end{proof}

\begin{corollary}\CorollaryName{univariate-list-decoder}
For any $\eps\in (0,1)$ and $\delta\in (0,1/n)$, there is an $(\eps,\delta)$-DP $L$-list-decodable learner for $\UnitGauss$ where 
\[
L = O\left( \frac{1}{(1-\gamma)^5 \alpha^2} \right),
\] 
and the algorithm uses 
\[
m_{\textsc{List}}(\alpha,\beta,\gamma,\eps,\delta) = O\left(\frac{\log(1/\beta\delta)}{(1-\gamma)^{2}\eps}+\frac{\log(1/(1-\gamma)\beta\delta)\sqrt{\log(1/(1-\gamma)\delta)}}{(1-\gamma)^{2}\eps}\right) = \widetilde{O}\left(\frac{\log^{3/2}(1/\beta\delta)}{(1-\gamma)^{2}\eps}\right)
\] 
samples.
\end{corollary}
\begin{proof}
The algorithm is simple; we run {\fontfamily{cmtt}\selectfont Univariate-Gaussian-Decoder}$(\alpha,\beta,\eps,\delta,\gamma,D)$ and obtain the sets $\what{M}$ and $\what{V}$. We then output $\what{\cF} = \{\cN(\widehat{\mu},\widehat{\sigma})\: : \: \widehat{\mu}\in \what{M},\: \widehat{\sigma}\in \what{V}\}$. The algorithm is $(\eps,\delta)$-DP by the guarantee of \Lemma{univariate-gaussian} and post processing (\Lemma{post-processing}). We have from the guarantee of \Lemma{univariate-gaussian} that 
\[
|\what{\cF}| = |\what{M}|\cdot|\what{V}| \leq \left(\frac{1728}{(1-\gamma)^{5}}\right)\cdot\left\lceil\log_{1+\alpha}(2)\right\rceil\cdot(2\left\lceil 1/\alpha\right\rceil+1).
\]
Note that $\log_{1+\alpha}(2) = \frac{\ln(2)}{\ln(1+\alpha)} \leq \frac{2\ln(2)}{\alpha}$ where the last inequality follows from the inequality $\ln(1+x) \geq x/2$ valid for $x \in [0,1]$.
This gives the claimed bound that $L = |\what{\cF}| = O\left( \frac{1}{(1-\gamma)^5 \alpha^2} \right)$.

For any $g\in\UnitGauss$ and $g'\in\cH_{\gamma}(g)$, given $n$ samples from $g'$ as input, we have from the guarantee of \Lemma{univariate-gaussian} and \Proposition{TV-univariate-gaussians} that the algorithm outputs $\what{\cF}$ satisfying $\TV(g,\what{\cF}) \leq \alpha$ so long as 
\[
n = \Omega\left(\frac{\log(1/\beta\delta)}{(1-\gamma)^{2}\eps}+\frac{\log(1/(1-\gamma)\beta\delta)\sqrt{\log(1/(1-\gamma)\delta)}}{(1-\gamma)^{2}\eps}\right) = \widetilde{\Omega}\left(\frac{\log^{3/2}(1/\beta\delta)}{(1-\gamma)^{2}\eps}\right).
\]
This proves the corollary.
\end{proof}

We can now use \Corollary{univariate-list-decoder} and \Theorem{reduction} to immediately get the following Theorem.
\begin{theorem}\TheoremName{univariate-pac-learner}
For any $\eps\in(0,1)$ and $\delta\in(0,1/n)$, there is an $(\eps,\delta)$-DP PAC learner for \emph{\kmix}$\left(\UnitGauss\right)$ that uses 
\begin{align*}
m(\alpha,\beta,\eps,\delta)&=\widetilde{O}\left(\frac{k^{2}\log^{3/2}(1/\beta\delta)}{\alpha^{2}\eps}\right)
\end{align*}
samples.
\end{theorem}

\section{Learning Mixtures of Axis-Aligned Gaussians}\label{sec:multivariate}

In this section, we prove the following result, which is a formal version of \Theorem{main-informal}.

\begin{theorem}\TheoremName{multivariate-pac-learner} For any $\eps\in(0,1)$ and $\delta\in(0,1/n)$, there is an $(\eps,\delta)$-DP PAC learner for \emph{\kmix}$\left(\dGauss\right)$ that uses 
\begin{align*}
m(\alpha,\beta,\eps,\delta)&=\widetilde{O}\left(\frac{k^{2}d\log^{3/2}(1/\beta\delta)}{\alpha^{2}\eps}\right)
\end{align*}
samples.
\end{theorem}

We now demonstrate how to construct an $(\eps,\delta)$-DP list-decodable learner for the class of $d$-dimensional axis-aligned Gaussians, $\dGauss$. Recall that the class of $d$-dimensional axis-aligned Gaussian is the class of all Gaussians with a diagonal covariance matrix, where the diagonals are arbitrary positive real numbers.

\begin{algorithm}
\caption{{\fontfamily{cmtt}\selectfont Multivariate-Gaussian-Decoder}$(\alpha,\beta,\gamma,\eps,\delta,D)$.}
\label{alg:multivariate-decoder}
\DontPrintSemicolon
\SetKwInOut{Input}{Input}\SetKwInOut{Output}{Output}
\setstretch{1.35}
\Input{Parameters $\eps,\alpha,\beta,\gamma \in (0,1)$, $\delta \in (0,1/n)$, and a dataset $D$}
\Output{Set of distributions $\what{\cF} \subset \dGauss$.}

Initialize $\what{V}_{j} \gets \emptyset$, $\what{M}_{j} \gets \emptyset$ for $j\in[d]$

Set $D_{i} \gets \{X_{i} \: : \: X \in D\}$ for $i \in [d]$ \tcp*{Split dataset by dimension.}

For $i \in [d]$ \textbf{do}

\quad $\what{M}_{i}$, $\what{V}_{i} \gets$ {\fontfamily{cmtt}\selectfont Univariate-Gaussian-Decoder}$(\alpha/d,\beta/d,\gamma,\eps/d,\delta/d,D_{i})$\label{alg:multivariate-decoder-line5}

$\what{M} \gets \{(\widehat{\mu}_{1},\dots,\widehat{\mu}_{d})\: : \: \widehat{\mu}_{i}\in \what{M}_{i},\: i\in[d] \}$

$\what{\Lambda} \gets \{\text{diag}(\widehat{\sigma}^{2}_{1},\dots,\widehat{\sigma}^{2}_{d})\: : \: \widehat{\sigma}_{i}\in \what{V}_{i},\: i\in[d] \}$

$\what{\cF} \gets \left\{\cN(\widehat{\mu},\widehat{\Sigma})\::\:\widehat{\mu}\in \what{M}, \: \widehat{\Sigma}\in \what{\Lambda}\right\}$\label{alg:multivariate-decoder-line6}

\textbf{Return} $\what{\cF}$
\end{algorithm}

\begin{lemma}\LemmaName{multivariate-list-decoder}
For any $\eps\in (0,1)$ and $\delta\in (0,1/n)$, Algorithm~\ref{alg:multivariate-decoder} is an $(\eps,\delta)$-DP $L$-list-decodable learner for $\dGauss$ where 
\[
L = O\left( \frac{d^2}{(1-\gamma)^5 \alpha^2} \right)^d,
\]
and the algorithm uses $$m_{\emph{\text{List}}}(\alpha,\beta,\gamma,\eps,\delta) = O\left(\frac{d\log(d/\beta\delta)}{(1-\gamma)^{2}\eps}+\frac{d\log(d/(1-\gamma)\beta\delta)\sqrt{\log(d/(1-\gamma)\delta)}}{(1-\gamma)^{2}\eps}\right) = \widetilde{O}\left(\frac{d\log^{3/2}(1/\beta\delta)}{(1-\gamma)^{2}\eps}\right)$$ samples.
\end{lemma}
\begin{proof}
~
\paragraph{Privacy.} We first prove the algorithm is $(\eps,\delta)$-DP. By the guarantee of \Lemma{univariate-gaussian}, each run of line~\ref{alg:multivariate-decoder-line5} in the loop is $(\eps/d,\delta/d)$-DP. No subsequent part of the algorithm accesses the data, so by post processing (\Lemma{post-processing}) and basic composition (\Lemma{composition}) the entire algorithm is $(\eps,\delta)$-DP. 

\paragraph{Bound on $|\what{\cF}|$.} We now prove the claimed upper bound on the size of $\what{\cF}$. By the guarantee of \Lemma{univariate-gaussian}, each $\what{M}_{i}$ and $\what{V}_{i}$ obtained on line~\ref{alg:multivariate-decoder-line5} satisfy $|\what{M}_{i}| \leq 144\cdot (2\cdot \lceil d/\alpha \rceil + 1)/(1-\gamma)^3$ and $|\what{V}_{i}| \leq 12\cdot\lceil \log_{1+\alpha/d}(2)\rceil/(1-\gamma)^{2}$. This immediately gives us \[|\widetilde{\cF}| = |\what{M}|\cdot|\what{\Lambda}| = \left(\prod_{i=1}^{d}|\what{M}_{i}|\right)\cdot\left(\prod_{i=1}^{d}|\what{V}_{i}|\right) \leq \left(\left(\frac{1728}{(1-\gamma)^{5}}\right)\cdot\left\lceil\log_{1+\alpha/d}(2)\right\rceil\cdot(2\cdot\left\lceil d/\alpha\right\rceil+1)\right)^{d}.\]
To get the bound on $L = |\what{\cF}|$ as stated in the lemma, we use the fact that $\log_{1+\alpha / d}(2) = \frac{\ln(2)}{\ln(1+\alpha / d)} \leq \frac{2\ln(2)}{\alpha / d}$, where the inequality uses the fact that $\ln(1+x) \geq x/2$ for $x \in [0,1]$.

\paragraph{Accuracy and sample complexity.} We now prove that the algorithm is a list-decodable learner. Fix some $g=\prod_{i=1}^{d}\cN(\mu_{i},\sigma_{i}^{2})\in \dGauss$ and $g'\in \cH_{\gamma}(g)$. By our choice of parameters and the guarantee of \Lemma{univariate-gaussian}, a single run of algorithm {\fontfamily{cmtt}\selectfont Univariate-Gaussian-Decoder} on line~\ref{alg:multivariate-decoder-line5} outputs lists $\what{M}_{i}$ and $\what{V}_{i}$ such that there exist $\widehat{\mu}_{i} \in \what{M}_{i}$ and $\widehat{\sigma}_{i} \in \what{V}_{i}$ satisfying $|\widehat{\mu}_{i}-\mu_{i}| \leq \alpha\sigma_{i}/d$ and $|\widehat{\sigma}_{i}-\sigma_{i}| \leq \alpha\sigma_{i}/d$ with probability at least $1-\beta/d$ so long as $$n = \Omega\left(\frac{d\log(d/\beta\delta)}{(1-\gamma)^{2}\eps}+\frac{d\log(d/(1-\gamma)\beta\delta)\sqrt{\log(d/(1-\gamma)\delta)}}{(1-\gamma)^{2}\eps}\right).$$ By a union bound, we have with probability no less than $1-\beta$ that for all $i\in [d]$, $|\widehat{\mu}_{i} - \mu_{i}| \leq \alpha\sigma_{i}/d$ and $|\widehat{\sigma}_{i}-\sigma_{i}| \leq \alpha\sigma_{i}/d$.
By a standard argument, this implies that with probability at least $1-\beta$ there is some $\what{g} \in \what{\cF}$ such that $\TV(\what{g}, g) \leq \alpha$ (see \Proposition{TV-univariate-gaussians} and \Proposition{TV-product-distribution}).
\end{proof}

We can now put together \Lemma{multivariate-list-decoder} and \Theorem{reduction} to immediately get \Theorem{multivariate-pac-learner}.
\bibliography{biblio}
\appendix
\section{Useful Facts}
\label{app:useful}
\begin{proposition}[Lemma 2.11, \cite{AshtianiBHLMP20}]\PropositionName{TV-univariate-gaussians}
For any $\mu,\widetilde{\mu} \in \bR$ and $\sigma,\wsigma>0$ with $|\widetilde{\mu}-\mu| \leq \alpha\sigma$ and $|\wsigma-\sigma| \leq \alpha\sigma$ where $\alpha \in [0,2/3]$, the Gaussians $\cN(\mu,\sigma^{2})$ and $\cN(\widetilde{\mu},\wsigma^{2})$ statisfy $$\TV\left(\cN(\mu,\sigma^{2}),\cN(\widetilde{\mu},\wsigma^{2})\right) \leq \alpha.$$
\end{proposition}

\begin{proposition}[Lemma 3.3.7, \cite{Reiss}]\PropositionName{TV-product-distribution}
For $i\in [d]$ let $p_{i}$ and $q_{i}$ be distributions over the same domain $\cX$. Then $$\TV\left(\prod_{i=1}^{d}p_{i},\prod_{i=1}^{d}q_{i}\right) \leq \sum_{i=1}^{d}\TV\left(p_{i},q_{i}\right).$$
\end{proposition}

\alphaNet*
\begin{proof}
    We will give an algorithmic proof of this fact.
    Let $r = \lceil 1/\alpha \rceil$ and fix $x \in \Delta_k$.
    Let $\ell = \sum_{i=1}^k rx_i - \lfloor rx_i \rfloor$.
    Note that $\sum_{i=1}^k rx_i = r$ and $rx_i - \lfloor rx_i \rfloor \in [0, 1)$ so $\ell$ is an integer in the interval $[0, r-1]$.
    Now define $\hat{x}$
    \[
        \hat{x}_i = \begin{cases}
            \frac{\lfloor rx_i \rfloor + 1}{r} & i \leq \ell \\
            \frac{\lfloor rx_i \rfloor}{r} & i > \ell
        \end{cases}.
    \]
    Clearly, $\|x-\hat{x}\|_{\infty} \leq 1/r \leq \alpha$.
    It remains to check that $\hat{x} \in \Delta_k$.
    Indeed,
    \[
        \sum_{i=1}^k \hat{x}_i
        = \sum_{i=1}^k \frac{\lfloor rx_i \rfloor}{r} + \frac{\ell}{r}
        = \sum_{i=1}^k \frac{\lfloor rx_i \rfloor}{r} + \sum_{i=1}^k \frac{rx_i - \lfloor rx_i \rfloor}{r}
        = 1,
    \]
    where in the second equality, we used the definition of $\ell$.
    Note that for each $i$, $\hat{x}_i \in \{0, 1/r, 2/r, \ldots, 1\}$ so this shows that
    \[
        \widehat{\Delta}_k =
        \{ ( t_1/r, \ldots, t_k/r ) \,:\, t \in \bZ^k_{\geq 0}, \|t\|_1 = r \},
    \]
    is an $\alpha$-net for $\Delta_k$ of size $(r+1)^k$.
    To obtain the bound as asserted in the claim, note that $r+1 = \lceil 1/\alpha \rceil +1 \leq 1/\alpha + 2 \leq 3/\alpha$ for $\alpha \in (0,1]$.
\end{proof}
\begin{lemma}[{Chernoff bound; see \cite[Exercise~2.3.6]{Vershynin18}}]
    \LemmaName{Chernoff}
    Let $X_1, \ldots, X_n$ be independent Bernoulli random variables. Let $S_n = \sum_{i=1}^n X_i$ and $\mu = S_n$. Then for any $\delta \in (0,1]$ and some absolute constant $c > 0$
    \[
        \Prob[|S_n - \mu| \geq \delta \mu] \leq 2e^{-c\mu \delta^2}.
    \]
\end{lemma}
\section{Locally Small Covers for Mixtures}\label{sec:impossibility}

To formally state and prove the impossibility result, we first introduce some useful definitions and results.
\begin{definition}[TV ball]
The total variation ball of radius $\gamma\in(0,1)$, centered at a distribution $g$ with respect to a set of distributions $\cF$, written $\ball{\gamma}{g}{\cF}$, is the following subset of $\cF$:
\begin{equation*}
\ball{\gamma}{g}{\cF}~\coloneqq~\left\{f \in \cF : \TV(g,\cF) \leq \gamma \right\}.
\end{equation*}
\end{definition}
In this paper we consider coverings and packings of sets of distributions with respect to the total variation distance.
\begin{definition}[$\gamma$-covers and $\gamma$-packings]
  For any $\gamma\in(0,1)$ a \emph{$\gamma$-cover} of a set of distributions $\mathcal{F}$ is a set of distributions $\mathcal{C}_\gamma$, such that for every $f \in \mathcal{F}$, there exists some $\widehat{f} \in \mathcal{C}_\gamma$ such that $\TV(f,\widehat{f}) \leq \gamma$.
  
  A \emph{$\gamma$-packing} of a set of distributions $\mathcal{F}$ is a set of distributions $\mathcal{P}_\gamma \subseteq \mathcal{F}$, such that for every pair of distributions $f, f' \in \mathcal{P}_\gamma$, we have that $\TV(f,f') \geq \gamma$.
\end{definition}
\begin{definition}[$\gamma$-covering and $\gamma$-packing number]
    For any $\gamma \in (0,1)$, the $\gamma$-\emph{covering number} of a set of distributions $\mathcal{F}$, $N(\mathcal{F},\gamma) := \min \{n\in\mathbb{N} : \exists \mathcal{C}_{\gamma} \emph{\text{ s.t. }} |\mathcal{C}_{\gamma}| = n\}$, is the size of the smallest possible $\gamma$-covering of $\mathcal{F}$. Similarly, the $\gamma$-\emph{packing number} of a set of distributions $\mathcal{F}$, $M(\mathcal{F},\gamma) := \max \{n\in\mathbb{N} : \exists \mathcal{P}_{\gamma}  \emph{\text{ s.t. }} |\mathcal{P}_{\gamma}| = n\}$, is the size of the largest subset of $\mathcal{F}$ that forms a packing for $\mathcal{F}$.
\end{definition}
The following Proposition follows directly from a well known relationship between packings and covers of metric spaces (see~\cite[Lemma 4.2.8]{Vershynin18}).
\begin{proposition}\PropositionName{pack-n-cover}
  For a set of distributions $\mathcal{F}$ with $\gamma$-covering number $M(\mathcal{F},\gamma)$ and $\gamma$-packing number $N(\mathcal{F},\gamma)$, the following holds:
  $$M(\mathcal{F},2\gamma) \leq N(\mathcal{F},\gamma) \leq M(\mathcal{F},\gamma).$$
\end{proposition}

We now formally define what it means for a set of distributions to be ``locally small''.
\begin{definition}[$\gamma$-locally small]
Fix some $\gamma \in (0,1)$. We say a set of distributions $\cF$ is $\gamma$-\emph{locally small} if $$\sup_{f\in \cF}|\ball{\gamma}{f}{\cF}| \leq k,$$ for some $k\in\mathbb{N}$. If no such $k$ exists, we say $\cF$ is \emph{not $\gamma$-locally small}.
\end{definition}

\begin{proposition}\PropositionName{impossibility}
For every $\gamma \in (0,1)$,  any $(\gamma/2)$-cover for $\twomix(\UnitGauss)$ is not $\gamma$-locally small.
\end{proposition}
\begin{proof}
Fix some $\gamma \in (0,1)$. Let $f = \cN(0,1)$ and define $g(\mu)\coloneqq (1-\gamma)\cN(0,1)+\gamma\cN(\mu,1)$ (note that $f = g(0)$). We will show that the following two statements hold for every $\mu, \mu' \in \bR$: 
\begin{enumerate}
    \item $\TV(g(\mu),g(\mu')) \leq \gamma$, and
    \item If $|\mu-\mu'| \geq C$ for a sufficiently large constant $C$, $\TV(g(\mu),g(\mu')) \geq \gamma/2$.
\end{enumerate}

Consider the set of distributions $ \cF = \{ g(\mu) \, : \, \mu \in \{C, 2C, \dots \} \}$ for some large positive constant $C$. For every $g,g' \in \cF$, it follows from claim 1 that $g,g' \in \ball{\gamma}{f}{\twomix(\UnitGauss)}$ and from claim 2 that $\TV(g,g')\geq \gamma/2$ for sufficiently large $C$. Thus, the $(\gamma/2)$-packing number of $\ball{\gamma}{f}{\twomix(\UnitGauss)}$ is unbounded, and by \Proposition{pack-n-cover}, the $(\gamma/2)$-covering number of $\ball{\gamma}{f}{\twomix(\UnitGauss)}$ is also unbounded. This implies that \emph{every} $(\gamma/2)$-cover for $\twomix(\UnitGauss)$ is not $\gamma$-locally small by definition.

It remains to prove the two claims above. From the definition of the TV distance we have
\begin{align}
    \TV(g(\mu),g(\mu')) &= \frac{1}{2}\left\|(1-\gamma)\cN(0,1)+\gamma\cN(\mu,1)-(1-\gamma)\cN(0,1)-\gamma\cN(\mu',1)\right\|_{1}\nonumber\\
    &=\frac{\gamma}{2}\left\|\cN(\mu,1)-\cN(\mu',1)\right\|_{1}\nonumber\\
    &=\gamma\TV(\cN(\mu,1),\cN(\mu',1)).\label{eq:tv-impossible}
\end{align}
Using the trivial upper bound on the TV distance between any two distributions, we have from Eq.~(\ref{eq:tv-impossible}) that $\TV(g(\mu),g(\mu')) \leq  \gamma$, which proves the first claim. If $|\mu -\mu'| \geq C$ for sufficiently large $C$, it follows from Gaussian tail bounds that $\TV(\cN(\mu,1),\cN(\mu',1)) = 1 - \exp(-\Omega(C^{2}))$. Thus, by choosing $C$ to be sufficiently large, it follows from Eq.~(\ref{eq:tv-impossible}) that $\TV(g(\mu),g(\mu')) \geq \gamma/2$.
\end{proof}
\section{Omitted Results from Section~\ref{sec:reduction}}
\begin{proposition}\PropositionName{TV-mixtures}
Let $\alpha \in (0,1)$ and $k \in \bN$.
Let $g = \sum_{i=1}^k w_i f_i$ and $\widetilde{g} = \sum_{i=1}^k \widetilde{w}_i \widetilde{f}_i$ be two mixture distributions that satisfy
\begin{enumerate}
\item $\|w - \widetilde{w}\|_{\infty} \leq \alpha / k$; and
\item $\TV(f_i, \widetilde{f}_i) \leq \alpha$ for $i \in [k]$ such that $w_i \geq \alpha / k$.
\end{enumerate}
Then $\TV(g, \widetilde{g}) \leq 3\alpha$.
\end{proposition}
\begin{proof}
Let $N = \{i\in[k]:w_{i}\geq \alpha/k\}$.
We have that
\begin{align*}
    \TV(\what{g},g) &=\frac{1}{2}\left\|\sum_{i=1}^{k}\what{w}_{i}\what{f}_{i}-\sum_{i=1}^{k}w_{i}f_{i} \right\|_{1}\\
    & = \frac{1}{2} \left\| \sum_{i=1}^k \what{w}_i (\what{f}_i - f_i) + \sum_{i=1}^k (\what{w}_i - w_i) f_i \right\|_1 \\
    &\leq\frac{1}{2} \left\|\sum_{i=1}^{k}\what{w}_{i}(\what{f}_{i}-f_{i})\right\|_{1} +\frac{1}{2}\left\|\sum_{i=1}^{k}(\what{w}_{i}-w_{i})f_{i}\right\|_{1}\\
    &\leq \frac{1}{2}\left\|\sum_{i\not\in N}\what{w}_{i}(\what{f}_{i}-f_{i})\right\|_{1} + \frac{1}{2}\left\|\sum_{i\in N}\what{w}_{i}(\what{f}_{i}-f_{i})\right\|_{1} +\frac{1}{2}\left\|\sum_{i=1}^{k}(\what{w}_{i}-w_{i})f_{i}\right\|_{1}\\
    &\leq\frac{1}{2}\sum_{i\not\in N} \what{w}_{i}\left\|\what{f}_{i}-f_{i}\right\|_{1} + \frac{1}{2}\sum_{i\in N}\what{w}_{i}\left\|\what{f}_{i}-f_{i}\right\|_{1} +\frac{1}{2}\sum_{i=1}^{k}|\what{w}_{i}-w_{i}|\left\|\what{f}_{i}\right\|_{1}\\
    &\leq\sum_{i\not\in N} \frac{\alpha}{k}\cdot 1 + \sum_{i\in N}\what{w}_{i}\cdot\alpha +\sum_{i=1}^{k}\frac{\alpha}{k}\cdot 1\\
    &\leq \alpha +\alpha +\alpha = 3\alpha.
\end{align*}
Note that in the second-to-last inequality, we used that for $i \notin N$, $\what{w}_i \leq \alpha / k$ and the trivial 
bound $\|\what{f}_i - f_i\|_1 \leq 2$ while for $i \in N$, we have $\|\what{f}_i - f_i\|_1 \leq \alpha$.
\end{proof}

\section{Omitted Results from Section~\ref{sec:univariate}}
\begin{proposition}\PropositionName{gaussian-mean-bin}
Fix some univariate Gaussian $g = \cN(\mu, \sigma^{2})$. Let $\wsigma$ satisfy $\sigma \leq \wsigma < 2\sigma$. Partition $\bR$ into disjoint bins $\{B_{i}\}_{i\in\bN}$ where $B_{i} = ((i-0.5)\wsigma,(i+0.5)\wsigma]$ and let $j = \lceil \mu/\wsigma \rfloor$, where $\lceil \cdot \rfloor$ denotes rounding to the nearest integer. It follows that:
\begin{enumerate}
    \item $\Prob_{X\sim g}[X \in B_{j}]  \geq 1/3$,
    \item $\mu \in [(j-0.5)\wsigma,(j+0.5)\wsigma]$.
\end{enumerate}
\end{proposition}
\begin{proof}
We first prove item 1.
\begin{align*}
    \mathbf{P}_{X\sim g}[X\in B_{j}] &= \Phi\left(\frac{(j+0.5)\wsigma}{\sigma}-\frac{\mu}{\sigma}\right) - \Phi\left(\frac{(j-0.5)\wsigma}{\sigma}-\frac{\mu}{\sigma}\right)\\
    &= \Phi\left(\frac{j\wsigma-\mu}{\sigma}+\frac{\wsigma}{2\sigma}\right) - \Phi\left(\frac{j\wsigma-\mu}{\sigma}-\frac{\wsigma}{2\sigma}\right)\\
    &\vcentcolon= f\left(\frac{j\wsigma-\mu}{\sigma}\right).
\end{align*}
Notice that $f(\xi) = \Phi(\xi+\wsigma/2\sigma)-\Phi(\xi-\wsigma/2\sigma)$ is decreasing with $|\xi|$. Furthermore, by the definition of $j$ we have, 
\begin{align*}
    \left|\frac{j\wsigma-\mu}{\sigma}\right| &= \frac{\wsigma}{\sigma}\left|j'-\frac{\mu}{\wsigma}\right|\\
    &\leq \frac{\wsigma}{\sigma}\cdot\frac{1}{2} = \frac{\wsigma}{2\sigma}.
\end{align*}
So, 
\begin{align*}
    \mathbf{P}_{X\sim g}[X\in B_{j}] &= f\left(\frac{j\wsigma-\mu}{\sigma}\right)\\
    &\geq f\left(\frac{\wsigma}{2\sigma}\right) \\
    & = \Phi\left(\frac{\wsigma}{\sigma}\right)-\Phi(0) \\
    &\geq \Phi(1) - \Phi(0) \geq 1/3,
\end{align*}
where the second last inequality follows from the fact that $\wsigma/\sigma \geq 1$ together with the monotonicity of the c.d.f. and the last inequality follows from a direct calculation.

We now prove the second claim that $\mu \in [(j-0.5)\wsigma, (j+0.5)\wsigma)]$. As we saw above, it follows that 
\begin{equation*}
    \frac{1}{\sigma}\left|j\wsigma-\mu\right| \leq \frac{\wsigma}{2\sigma}
    \implies \mu \in [(j-0.5)\wsigma,(j+0.5)\wsigma].
\end{equation*}
\end{proof}
\begin{proposition}\PropositionName{gaussian-variance-bin}
Fix some univariate Gaussian $g = \cN(0,\sigma^{2})$. Partition $\bR_{>0}$ into disjoint bins $\{B_{i}\}_{i\in\bZ}$ where $B_{i} = (2^{i},2^{i+1}]$ and let $j\in\bN$ satisfy $2^{j} < \sigma \leq 2^{j+1}$. It follows that: $$\Prob_{X \sim g}[|X| \in B_{j}] \geq \frac{1}{4}.$$
\end{proposition}
\begin{proof}
Since $2^{j} < \sigma \leq 2^{j+1}$, we can write $\sigma = 2^{j+c}$ for some $c \in (0,1]$. Let $x = 2^{-c}$ and notice $x \in [1/2,1)$. We have the following:
\begin{align}
    \mathbf{P}_{X\sim g}[|X|\in B_{j}] &= 2\left(\Phi\left(\frac{2^{j+1}}{\sigma}\right)-\Phi\left(\frac{2^{j}}{\sigma}\right)\right)\nonumber\\
    &= 2\left(\Phi\left(2^{1-c}\right)-\Phi\left(2^{-c}\right)\right)\nonumber\\
    &= 2f(2^{-c}),\label{eq:gaussian-variance-bin}
\end{align}
where we define $f(x) = \Phi(2x)-\Phi(x)$. We now aim to lower bound $f(x)$. By taking the derivative of $f(x)$ twice, we have that $f''(x) =\sqrt{(1/2\pi)}(x\text{exp}(-x^{2}/2)- 8 x\text{exp}(-2x^{2}))$. By a simple calculation, we have that $f''(x) \leq 0$ when $x\in[0,2\ln 8/3] \supset [1/2,1)$, so $f(x)$ is concave when $x \in [1/2,1)$. This implies that $f(x) \geq \min\{f(1/2),f(1)\}$ for any $x\in[1/2,1)$, so from Eq.~(\ref{eq:gaussian-variance-bin}) we have
\begin{align*}
   \mathbf{P}_{X\sim g}[|X|\in B_{j}] &\geq 2\min\left\{f(1/2),f(1)\right\}\\
   &= 2 \min\left\{\Phi\left(1\right)-\Phi\left(\frac{1}{2}\right),\Phi\left(2\right)-\Phi\left(1\right)\right\}\\
    &>\frac{1}{4},
\end{align*}
where the last inequality follows from a direct calculation.
\end{proof}
\begin{proposition}
\PropositionName{centered-samples-prob}
    Fix $g = \cN(\mu, \sigma^2)$ and $g' \in \cH_{\gamma}(g)$.
    Let $Z = (X_1 - X_2)/\sqrt{2}$ where $X_1, X_2 \sim g'$ i.i.d. Let $Y \sim \cN(0, \sigma^2)$.
    Then for any measurable $S \subseteq \bR$
    \[
        \Prob[|Z| \in S]
        \geq (1-\gamma)^2 \cdot \Prob[|Y| \in S].
    \]
\end{proposition}
\begin{proof}
    We prove this via a coupling argument.
    Since $g' \in \cH_{\gamma}(g)$ we have $g' = (1-\gamma)g + \gamma h$ for some distribution $h$.

    Let $Y_1, Y_2 \sim g$ i.i.d.~so that $Y = \frac{Y_1-Y_2}{\sqrt{2}} \sim \cN(0, \sigma^2)$.
    Also, let $H_1, H_2 \sim h$ i.i.d.
    Finally, let $B_1, B_2$ be independent Bernoulli random variables with parameter $1-\gamma$, i.e.~$B_i = 1$ with probability $1-\gamma$ and $B_i = 0$ with probability $\gamma$.
    
    Now let $X_i = Y_i \cdot B_i + H_i \cdot (1-B_i)$ and note that $X_i \sim g'$.
    If $B_1 = B_2 = 1$ and $|Y| \in S$ then certainly $|Z| = |X_1 - X_2| / \sqrt{2} \in S$.
    Hence,
    \[
        \Prob[|Z| \in S]
        \geq \Prob[\{B_1 = 1\} \cap \{B_2 = 1\} \cap \{|Y| \in S\}]
        = (1-\gamma)^2 \Prob[|Y| \in S],
    \]
    where the last equality uses the fact that $B_1, B_2, Y$ are mutually independent random variables.
\end{proof}

\section{Learning Mixtures of Gaussians with Known Covariance}\label{app:identity-gaussians}

In this section, we prove the following result, which is a formal version of \Theorem{identity-informal}. Let $\UnitGauss_{1}^d$ be the class of Gaussians with identity covariance matrix.

\begin{theorem}\TheoremName{identity-pac-learner} For any $\eps\in(0,1)$ and $\delta\in(0,1/n)$, there is an $(\eps,\delta)$-DP PAC learner for \emph{\kmix}$\left(\cG_1^d \right)$ that uses 
\begin{align*}
m(\alpha,\beta,\eps,\delta) &= \widetilde{O}\left(\frac{kd \log(1/\beta)}{\alpha^2} + \frac{kd+\log(1/\beta\delta)}{\alpha\eps}\right)
\end{align*}
samples.
\end{theorem}

Note that the theorem also implies the case where the covariance matrix $\Sigma$ is an arbitrary but known covariance matrix.
Indeed, given samples $X_1, \ldots, X_m$, one can apply the algorithm of \Theorem{identity-pac-learner} to $\Sigma^{-1/2} X_1, \ldots, \Sigma^{-1/2} X_m$ instead.

The proof of \Theorem{identity-pac-learner} follows from \Theorem{reduction} and \Corollary{identity-covariance}, which is a corollary of \Lemma{univariate-mean}. 

\begin{corollary}
\CorollaryName{identity-covariance}
For any $\eps\in (0,1)$ and $\delta\in (0,1/n)$, there is an $(\eps,\delta)$-DP $L$-list-decodable learner for $\UnitGauss_{1}^d$
where $L = O(d/(1-\gamma)\alpha)^d$, and the number of samples used is 
\[
    m_{\textsc{List}}(\alpha,\beta,\gamma,\eps,\delta)
    = O\left(\frac{d\log(d/\beta\delta)}{(1-\gamma)\eps}\right).
\]
\end{corollary}
\begin{proof}
For each $i \in [d]$ let $D_i = \{X_i \,:\, X \in D\}$ be the dataset consisting of the $i$th coordinate of each element in $D$.
We run {\fontfamily{cmtt}\selectfont Univariate-Mean-Decoder}$(\eps/d,\delta/d,\beta/d,\gamma,\sigma,D_i)$ to obtain the set $\wtilde{M}_i$.
Let $\what{M}_i$ be an $\alpha/d$-net of the set of intervals $\{[\widetilde{\mu}_i-1,\widetilde{\mu}_i+1] \,:\, \wtilde{\mu}_i \in \wtilde{M}_i\}$ of size $|\wtilde{M}_i| \cdot (2 \cdot \lceil d/2\alpha \rceil + 1)$, i.e.~
\[
    \what{M}_i = \{ \widetilde{\mu}_i + 2j\alpha/d \,:\, \widetilde{\mu}_i \in 
    \wtilde{M}_i, \, j \in \{0, \pm 1, \ldots, \pm \lceil d/2\alpha \rceil \}.
\]
Let $\what{M} = \{ (\what{\mu}_1, \ldots, \what{\mu}_d) \,:\, \what{\mu}_i \in \what{M}_i \}$.
We then return $\what{\cF} = \{ \cN(\what{\mu}, I) \,:\, \widehat{\mu} \in \what{M}\}$.
Finally, \Lemma{univariate-mean} (with a union bound over the $d$ coordinates), basic composition (\Lemma{composition}), and post-processing (\Lemma{post-processing}) imply that the algorithm is $(\eps, \delta)$-DP while \Lemma{univariate-mean}, \Proposition{TV-product-distribution}, and \Proposition{TV-univariate-gaussians} imply the accuracy guarantee.
\end{proof}
\end{document}